\documentclass{article}

\PassOptionsToPackage{numbers}{natbib}
    \usepackage[final]{neurips_2021}

\usepackage[utf8]{inputenc} % allow utf-8 input
\usepackage[T1]{fontenc}    % use 8-bit T1 fonts
\usepackage{hyperref}       % hyperlinks
\usepackage{url}            % simple URL typesetting
\usepackage{booktabs}       % professional-quality tables
\usepackage{amsfonts}       % blackboard math symbols
\usepackage{nicefrac}       % compact symbols for 1/2, etc.
\usepackage{microtype}      % microtypography
\usepackage{xcolor}         % colors

\usepackage{bbm}
\usepackage{enumitem}
\usepackage{wrapfig}
\usepackage{amsthm}

\usepackage{amsthm}
\usepackage{graphicx}
\usepackage{wrapfig}

\usepackage{amssymb}

\usepackage{color}
\usepackage{mathtools}

\usepackage{multirow}

\usepackage[noend]{algpseudocode}
\usepackage{algorithm}

\usepackage[most]{tcolorbox}

\usepackage[scaled=0.84]{helvet}

\newcommand{\E}{\mathop{\mathbb{E}}} % behaves like a correct operator 
\newcommand{\MD}{S}
\newcommand{\momcov}{\Sigma}

\newcommand{\T}{\mathcal{T}} % maybe we change this if we want
\newcommand{\unnorm}{\bar} % you can change this back if you want
\newcommand{\ELBO}{\mathrm{ELBO}} % you can change this back if you want

\newcommand{\UHA}{UHA}

\usepackage{amsthm}
\usepackage{thmtools} 
\usepackage{thm-restate}

\declaretheorem{theorem}
\declaretheorem[sibling=theorem]{lemma}

\hypersetup{%
    pdfborder = {0 0 0},
    colorlinks,
    citecolor=blue,
    linkcolor=blue,
}

\title{MCMC Variational Inference via Uncorrected Hamiltonian Annealing}

\author{
  Tomas Geffner \\
  College of Information and Computer Science\\
  University of Massachusetts, Amherst\\
  Amherst, MA \\
  \texttt{tgeffner@cs.umass.edu} \\
  \And
  Justin Domke \\
  College of Information and Computer Science\\
  University of Massachusetts, Amherst\\
  Amherst, MA \\
  \texttt{domke@cs.umass.edu} \\
}

\begin{document}

\maketitle

\begin{abstract}
  Given an unnormalized target distribution we want to obtain approximate samples from it and a tight lower bound on its (log) normalization constant $\log Z$. Annealed Importance Sampling (AIS) with Hamiltonian MCMC is a powerful method that can be used to do this. Its main drawback is that it uses non-differentiable transition kernels, which makes tuning its many parameters hard. We propose a framework to use an AIS-like procedure with Uncorrected Hamiltonian MCMC, called Uncorrected Hamiltonian Annealing. Our method leads to tight and differentiable lower bounds on $\log Z$. We show empirically that our method yields better performances than other competing approaches, and that the ability to tune its parameters using reparameterization gradients may lead to large performance improvements.
\end{abstract}

%!TEX root = ../main.tex

\section{Introduction} \label{sec:intro}

Variational Inference (VI) \cite{blei2017variational, wainwright2008graphical, zhang2017advances} is a method to do approximate inference on a target distribution $p(z) = \unnorm p(z) / Z$ that is only known up to the normalization constant $Z$. The basic insights are, first, that the evidence lower bound (ELBO) $\mathbb{E}_{q(z)}[\log \unnorm p(z) - \log q(z)]$ lower-bounds $\log Z$ and, second, that maximizing the ELBO is equivalent to minimizing the KL-divergence from $q$ to $p$. The simplest VI method chooses a parameterized family for $q$ and optimizes its parameters to maximize the ELBO.

A recent direction involves combining VI with Markov chain Monte Carlo (MCMC) \cite{salimans2015markov, wolf2016variational}. These methods can be seen as an instance of the auxiliary VI framework \cite{agakov2004auxiliary} -- they create an augmented variational distribution that represents all intermediate random variables generated during the MCMC procedure. An augmented target distribution that attempts to capture the inverse MCMC dynamics is optimized jointly with this variational distribution. However, it has been observed that capturing inverse dynamics is challenging \cite[\S5.4]{wolf2016variational} (further discussion in Section \ref{sec:relwork}).

Annealed Importance Sampling (AIS) \cite{jarzynski1997equilibrium, neal2001ais} is a powerful technique used to build augmented distributions without the need of learning inverse dynamics. While it was originally proposed to estimate expectations using importance sampling, it can be easily used to build lower bounds on normalization constants of intractable densities \cite{grosse2015sandwiching, wu2016quantitative}. AIS creates a sequence of densities that bridge from a tractable initial approximation $q$ to the target $\unnorm p$. Then, the augmented variational distribution is given by a sequence of MCMC kernels targeting each bridging density, while the augmented target uses the {\em reversals} of those kernels. It turns out that the ratio of these augmented distributions can be computed using only evaluations of the bridging densities. Combining Hamiltonian MCMC kernels with AIS has been observed to produce strong lower bounds \cite{sohl2012hamiltonian, wu2016quantitative}.

However, these bounds are sensitive to numerous parameters, such as the initial distribution, bridging schedule, and parameters of the MCMC kernels. It would be desirable to optimize these parameters to tighten the bound. Unfortunately, the presence of Metropolis-Hastings acceptance steps means that the the final estimator is non-differentiable, and thus reparameterization gradients cannot be used.

In this work, we propose {\em Uncorrected} Hamiltonian Annealing (\UHA), a differentiable alternative to Hamiltonian AIS. We define an augmented variational distribution using Hamiltonian MCMC kernels, but dropping the accept-reject steps. This is motivated by the fact that Hamiltonian dynamics sometimes have high acceptance rates. Since these uncorrected MCMC kernels do not exactly hold the bridging densities invariant, an augmented target distribution cannot be defined in terms of reversals. Instead, we define our augmented target by deriving an algorithm for the exact reversal of the original (corrected) MCMC kernel and dropping the accept-reject step. Surprisingly, this yields a very simple expression for the resulting lower bound.
% Surprisingly, this yields a simple expression for the lower bound that only requires evaluating the momentum density at intermediate points in the procedure.

We use reparameterization gradients to tune various parameters involved in the lower bound produced by \UHA, including the initial approximation $q$, parameters of the uncorrected MCMC kernel, and the bridging densities. Experimentally, tuning all these leads to large gains. For example, in several inference tasks we observe that tuning \UHA\ with $K = 64$ bridging densities gives better results than traditional Hamiltonian AIS with $K = 512$.

Finally, we use \UHA\ to train VAEs \cite{vaes_welling, rezende2014stochastic}. In this case we observe that using \UHA\ leads to higher ELBOs. In addition, we observe that increasing the number of bridging densities with \UHA\ consistently leads to better results, and that for a large enough number of bridging densities the variational gap (difference between ELBO and true log-likelihood) becomes small, and models with higher log-likelihood are obtained.

%!TEX root = ../main.tex

% \vspace{-0.2cm}
\section{Preliminaries} \label{sec:preliminaries}
% \vspace{-0.2cm}

\textbf{Variational inference and augmentation.} Suppose that $p(z) = \frac{1}{Z}\unnorm p(z)$ is some target density, where $\unnorm p$ is unnormalized and $Z=\int \unnorm p(z) dz$ is the corresponding normalizer, and let
\begin{equation}
\ELBO(q(z), \unnorm p(z)) = \E_{q(z)} \log \frac{\unnorm p(z)}{q(z)}
\end{equation}
be the "ELBO operator". Variational inference (VI) is based on the fact that for any $q(z)$ we have \cite{blei2017variational}
\begin{equation}
\log Z = \ELBO(q(z), \unnorm p(z)) + \mathrm{KL}(q(z) \Vert p(z)).
\end{equation}
In VI, the parameters of $q$ are tuned to maximize the "evidence lower bound" (ELBO). Since the KL-divergence is non-negative, this is always a lower bound on $\log Z$. Also, maximizing the ELBO is equivalent to minimizing the KL-divergence from $q$ to $p$.

To get tighter bounds and better approximations recent work has made use of {\em augmented} distributions \cite{agakov2004auxiliary, huang2018improving}. Let $z_{1:M}=(z_1, \cdots, z_M)$
%be a set of "replicates" \tg{huh?} of $z$ 
and suppose that $\unnorm p(z_{1:M}) = \unnorm p(z_M) p(z_{1:M-1}\vert z_M)$ augments the original target density while preserving its normalization constant. Then, for any $q(z_{1:M})$ we have 
% \begin{equation}
% \log Z = \E_{q(z_{1:M})} \log \frac{\unnorm p(z_{1:M})}{q(z_{1:M})} + \mathrm{KL}(q(z_{1:M}) \Vert p(z_{1:M})). \label{eq:augELBO}
% \end{equation}
\begin{equation}
\log Z = \ELBO(q(z_{1:M}), \unnorm p(z_{1:M})) + \mathrm{KL}(q(z_{1:M}) \Vert p(z_{1:M})). \label{eq:augELBO}
\end{equation}
The first term is called the "augmented" ELBO and again lower bounds $\log Z$. By the chain rule of KL-divergence \cite{cover1999elements}, the KL-divergence from $q$ to $p$ over $z_{1:M}$ upper-bounds the KL-divergence over $z_M$. This justifies using the marginal of $q$ over $z_M$ to approximate the original target distribution.

% with respect to $w$. A typical choice for the estimator is $R_w(z) = \unnorm p(z) / q_w(z)$. While eq.~\ref{eq:boundVI} is quite simple, it represents the foundation of most VI methods. It is used when the variational family is defined as with Normalizing Flows \cite{tabak2013family, rezende2015variational, dinh2016density}, and when more powerful unbiased estimators $R_w$ are used \cite{agakov2004auxiliary, salimans2015markov, wolf2016variational, caterini2018hamiltonian, IWVAE, domke2018importance, domke2019divide}.

% \jd{I changed the order of this section to first introduce T and U and only then introduce the bridging densities. What do you think?}
\textbf{Annealed Importance Sampling.} A successful approach for creating augmented distributions is Annealed Importance Sampling (AIS) \cite{neal2001ais}. It creates an augmented proposal distribution $q$ by applying a sequence of transition densities $T_m(z_{m+1} \vert z_m)$, and an augmented target by defining transition densities $U_m(z_m \vert z_{m+1})$. This gives the augmented densities
\begin{equation}
q(z_{1:M}) = q(z_1) \prod_{m=1}^{M-1} T_m(z_{m+1} \vert z_m) \quad \mbox{ and } \quad \unnorm p(z_{1:M}) = \unnorm p(z_M) \prod_{m=1}^{M-1} U_m(z_m \vert z_{m + 1}).
\end{equation}
Naively, the ratio of these densities is
\begin{equation}
\frac{\unnorm p(z_{1:M})}{q(z_{1:M})} = \frac{\unnorm p(z_M)}{q(z_1)} \prod_{m=1}^{M-1} \frac{U_m(z_m\vert z_{m+1})}{T_m(z_{m+1}\vert z_m)}. \label{eq:AISnormal}
\end{equation}
To define the transitions $T_m$ and $U_m$, AIS creates a sequence of unnormalized densities $\unnorm \pi_1, \hdots, \unnorm \pi_{M-1}$ that ``bridge'' from a starting distribution $q$ to the target $\unnorm p$, meaning that $\unnorm \pi_1$ is close to $q$ and $\unnorm \pi_{M-1}$ is close to $\unnorm p$. %For each intermediate distribution, $T_m(z_{m+1} \vert z_m)$ is chosen to be a transition density that holds $\pi_m$ invariant, typically a Markov kernel. In general, such transition densities are often intractable to evaluate, due to the difficulty of evaluating the total probability that a Metropolis-Hastings method will not accept a move. AIS solves this problem by choosing $U_m$ as the reversal of $T_m$ with respect to $\pi_m$, defined as
Then, for each intermediate distribution, $T_m(z_{m+1} \vert z_m)$ is chosen to be a Markov kernel that holds $\pi_m$ invariant, and $U_m$ to be the reversal of $T_m$ with respect to $\pi_m$, defined as
\begin{equation}
U_m(z_m \vert z_{m+1}) = T(z_{m+1} \vert z_m) \frac{\pi_m (z_m)}{\pi_m(z_{m+1})}.
% U_m(z_m \vert z_{m+1}) = T(z_{m+1} \vert z_m) \nicefrac{\pi_m (z_m)}{\pi_m(z_{m+1})}.
\end{equation}
This choice produces a simplification so that eq.~\ref{eq:AISnormal} becomes
\begin{equation}
\frac{\unnorm p(z_{1:M})}{q(z_{1:M})} = \frac{\unnorm p(z_M)}{q(z_1)} \prod_{m=1}^{M-1} \frac{\unnorm \pi_m(z_m)}{\unnorm \pi_{m}(z_{m+1})}. \label{eq:AISest}
\end{equation}
% This can be evaluated without needing to evaluate the transition densities. Research has shown that the AIS augmentation may lead to extremely tight lower bounds $\ELBO(q(z_{1:M}), \unnorm p(z_{1:M}))$ \cite{grosse2015sandwiching, sohl2012hamiltonian, wu2016quantitative}.
This can be easily evaluated without needing to evaluate the transition densities. The ratio from eq.~\ref{eq:AISest} can be used to get an expression for the lower bound $\ELBO(q(z_{1:M}), \unnorm p(z_{1:M}))$. Research has shown that the AIS augmentation may lead to extremely tight lower bounds \cite{grosse2015sandwiching, grosse2016measuring, sohl2012hamiltonian, wu2016quantitative}.

\textbf{Hamiltonian Dynamics.} Many MCMC methods used to sample from $p(z)$ are based on Hamiltonian dynamics \cite{betancourt2017geometric, chen2014stochastic, neal2011mcmc, welling2011bayesian}. The idea is to create an augmented distribution $p(z, \rho) = p(z) \MD(\rho)$, where $\MD(\rho)$ is a distribution over a momentum variable $\rho$ (e.g. a Multivariate Gaussian). Then, one can define numerical integration schemes where $z$ and $\rho$ evolve while nearly holding $p(z, \rho)$ constant. When corrected by a Metropolis-Hastings acceptance step, this can be made to exactly hold $p(z, \rho)$ invariant. This is alternated with a scheme that resamples the momentum $\rho$ while holding $\MD(\rho)$ invariant. When Hamiltonian dynamics work well, $z$ can quickly move around, suppressing random-walk behavior.

There are a variety of different Hamiltonian MCMC methods, corresponding to different integration schemes, momentum distributions, and ways of resampling the momentum. For instance, HMC and Langevin dynamics use the leapfrog integrator, a Gaussian for the momentum variables and a full resampling of the momentum variables at each step \cite{neal2011mcmc, welling2011bayesian}. On the other hand, if the momentum variables are only partially resampled, the under-damped variants of HMC and Langevin dynamics are recovered \cite{neal2011mcmc}. It was observed that partial resampling may lead to improved perfomance \cite{cheng2018underdamped}.

It is easy to integrate Hamiltonian dynamics into AIS. First, define an augmented target $\unnorm p(z,\rho)=\unnorm p(z) \MD(\rho)$ and an augmented starting distribution $q(z,\rho)=q(z)\MD(\rho)$. Then, create a series of augmented densities $\unnorm \pi_1(z, \rho), \hdots, \unnorm \pi_{M-1}(z, \rho)$ bridging the two as $\unnorm \pi_m(z,\rho) = \unnorm \pi_m(z) S(\rho)$. Finally, define the forward transition $T_m(z_{m+1}, \rho_{m+1} \vert z_m, \rho_m)$ to be an iteration of a Hamiltonian MCMC method that leaves $\pi_m(z, \rho)$ invariant. We will describe a single transition $T_m$ as a sequence of three steps: (1) resample the momentum; (2) simulate Hamiltonian dynamics and apply an accept-reject step; and (3) negate the momentum. The precise process that defines the transition is shown in Alg.~\ref{alg:correctedtm}. Note that this algorithm is quite general, and compatible with HMC, Langevin dynamics and their underdamped variants (by selecting an appropriate integrator and resampling method).

\begin{algorithm}[ht]
\caption{Corrected $T_m(z_{m+1}, \rho_{m+1} \vert z_m, \rho_m )$}
\label{alg:correctedtm}
\begin{algorithmic}
\State 1. Sample $\rho'_m$ from some $s(\rho'_m\vert \rho_m)$ that leaves $\MD(\rho)$ invariant. Set $z'_m \leftarrow z_m$.
% \State \hspace{0.3cm} Set $z'_m \leftarrow z_m$.
\vspace{0.08cm}
\State 2. Simulate Hamiltonian dynamics as \small$(z''_m, \rho''_m) \leftarrow \T_m(z'_m, \rho'_m)$\normalsize.
\State \hspace{0.3cm} Calculate an acceptance probability \small$\alpha = \min\left(1, \unnorm \pi_m(z''_m, \rho''_m) / \unnorm \pi_m(z'_m, \rho'_m)\right)$\normalsize.
\State \hspace{0.3cm} With probability $\alpha$, set \small$(z'''_m,\rho'''_m) \leftarrow (z''_m,\rho''_m)$\normalsize. Otherwise, set \small$(z'''_m, \rho'''_m) \leftarrow (z'_m, \rho'_m)$\normalsize.
\vspace{0.08cm}
\State 3. Reverse the momentum as \small$(z_{m+1}, \rho_{m+1}) \leftarrow (z'''_m, -\rho'''_m)$\normalsize.
\State \Return $(z_{m+1}, \rho_{m+1})$
\end{algorithmic}
\end{algorithm}

Representing $T_m$ this way makes it easy to show it holds the density $\pi_m(z,\rho)$ invariant. The overall strategy is to show that each of the steps 1-3 holds $\pi_m$ invariant, and so does the composition of them \cite[\S3.2]{neal2011mcmc}. For steps 1 and 3 this is trivial, provided that $\MD(\rho)=\MD(-\rho)$. For step 2, we require that the simulation $\T_m$ has unit Jacobian and satisfies $\T_m^{-1} = \T_m$. Then, $\T_m$ can be interpreted as a symmetric Metropolis-Hastings proposal, meaning the Metroplis-Hastings acceptance probability $\alpha$ is as given. A typical choice for $\T_m$ that satisfies these requirements is the leapfrog integrator with a momentum reversal at the end. (This reversal then gets "un-reversed" in step 3 for accepted moves.)

% Representing $T_m$ this way makes it easy to show it holds the target $\pi_m(z,\rho)$ invariant. The overall strategy is to show that each of the steps 1-3 holds $\pi_m$ invariant, and so does the composition of them \cite[\S3.2]{neal2011mcmc}. The only requirements are a symmetric momentum distribution ($\MD(\rho)=\MD(-\rho)$) and a dynamics simulator $\T_m$ with unit Jacobian and self-inverting ($\T_m^{-1} = \T_m$). A typical choice for $\T_m$ that satisfies these requirements is the leapfrog integrator with a momentum reversal at the end. (This reversal then gets "un-reversed" in step 3 for accepted moves.)

Since $T_m$ holds $\pi_m$ invariant, we can define $U_m$ as the reversal of $T_m$ wrt $\pi_m$. Then, eq.~\ref{eq:AISest} becomes
\begin{equation}
\frac{\unnorm p(z_{1:M}, \rho_{1:M})}{q(z_{1:M}, \rho_{1:M})} = \frac{\unnorm p(z_M, \rho_M)}{q(z_1, \rho_1)} \prod_{m=1}^{M-1} \frac{\unnorm \pi_m(z_m, \rho_m)}{\unnorm \pi_m(z_{m+1}, \rho_{m+1})}. \label{eq:haisest}
\end{equation}
Using this ratio we get an expression for the lower bound $\ELBO(q(z_{1:M}, \rho_{1:M}), \unnorm p(z_{1:M}, \rho_{1:M}))$ obtained with Hamiltonian AIS. While this method has been observed to yield strong lower bounds on $\log Z$ \cite{sohl2012hamiltonian, wu2016quantitative} (see also Section \ref{sec:exps_real}), its performance depends on many parameters: initial distribution $q(z)$, momentum distribution $S$, momentum resampling scheme, simulator $\mathcal{T}_m$, and bridging densities. We would like to tune these parameters by maximizing the ELBO using reparameterization-based estimators. However, due to the accept-reject step required by the Hamiltonian MCMC transition, the resulting bound is not differentiable, and thus reparameterization gradients are not available.

%!TEX root = ../main.tex

% \vspace{-0.2cm}
\section{Uncorrected Hamiltonian Annealing} \label{sec:uhais}
% \vspace{-0.2cm}

The contribution of this paper is the development of {\em uncorrected} Hamiltonian Annealing (\UHA). This method is similar to Hamiltonian AIS (eq. \ref{eq:haisest}), but yields a differentiable lower bound. The main idea is simple. For any transitions $T_m$ and $U_m$, by the same logic as in eq.~\ref{eq:AISnormal}, we can define the ratio
\begin{equation}
\frac{\unnorm p(z_{1:M}, \rho_{1:M})}{q(z_{1:M}, \rho_{1:M})} = \frac{\unnorm p(z_M, \rho_M)}{q(z_1, \rho_1)} \prod_{m=1}^{M-1} \frac{U_m(z_m, \rho_m\vert z_{m+1}, \rho_{m+1})}{T_m(z_{m+1}, \rho_{m+1} \vert z_m, \rho_m)}.
\label{eq:haisobv}
\end{equation}
Hamiltonian AIS defines $T_m$ as a Hamiltonian MCMC kernel that holds $\pi_m$ invariant, and $U_m$ as the reversal of $T_m$ with respect to $\pi_m$. While this leads to a nice simplification, there is no {\em requirement} that these choices be made. We can use {\em any} transitions as long as the ratio $U_m / T_m$ is tractable.

We propose to use the "uncorrected" versions of the transitions $T_m$ and $U_m$ used by Hamiltonian AIS, obtained by dropping the accept-reject steps. To get an expression for the uncorrected $U_m$ we first derive the reversal $U_m$ used by Hamiltonian AIS (Alg.~\ref{alg:correctedum}). These uncorrected transitions are no longer reversible with respect to the bridging densities $\pi_m(z, \rho)$, and thus we cannot use the simplification used by AIS to get eq.~\ref{eq:haisest}. Despite this, we show that the ratio $U_m / T_m$ for the uncorrected transitions can still be easily computed (Thm. \ref{thm:ratioUT}). This produces a differentiable estimator, meaning the parameters can be tuned by stochastic gradient methods designed to maximize the ELBO.

We start by deriving the process that defines the transition $U_m$ used by Hamiltonian AIS. This is shown in Alg.~\ref{alg:correctedum}. It can be observed that $U_m$ follows the same three steps of $T_m$ (resample momentum, Hamiltonian simulation with accept-reject, momentum negation), but in reverse order.

\begin{algorithm}[ht]
\caption{Corrected $U_m(z_m, \rho_m \vert z_{m+1}, \rho_{m+1})$}
\label{alg:correctedum}
\begin{algorithmic}
\State 1. Set \small$(z'''_m, \rho'''_m) \leftarrow (z_{m+1}, -\rho_{m+1})$\normalsize.
\vspace{0.08cm}
\State 2. Simulate Hamiltonian dynamics as \small$(z''_m, \rho''_m) \leftarrow \T_m(z'''_m, \rho'''_m)$\normalsize.
\State \hspace{0.3cm} Calculate an acceptance probability \small$\alpha = \min\left(1, \unnorm \pi_m(z''_m, \rho''_m) / \unnorm \pi_m(z'''_m, \rho'''_m)\right)$\normalsize.
\State \hspace{0.3cm} With probability $\alpha$, set \small$(z'_m,\rho'_m) \leftarrow (z''_m,\rho''_m)$\normalsize. Otherwise, set \small$(z'_m, \rho'_m) \leftarrow (z'''_m, \rho'''_m)$\normalsize.
\vspace{0.08cm}
\State 3. Sample $\rho_m$ from $s_\mathrm{rev}(\rho_m\vert \rho'_m)$, the reversal of $s(\rho'_m\vert \rho_m)$ with respect to $\MD(\rho_m)$. Set $z_m \leftarrow z'_m$.
\State \Return $(z_{m}, \rho_{m})$
\end{algorithmic}
\end{algorithm}

\begin{lemma} \label{lemma:reversal}
The corrected $U_m$ (Alg.~\ref{alg:correctedum}) is the reversal of the corrected $T_m$ (Alg.~\ref{alg:correctedtm}) with respect to $\pi_m$.
\end{lemma}
\vspace{-0.3cm}
\begin{proof}[(Proof Sketch)] 
First, we claim the general result that if $T_1$, $T_2$ and $T_3$ have reversals $U_1$, $U_2$ and $U_3$, respectively, then the composition $T = T_1 \circ T_2 \circ T_3$ has reversal $U = U_3 \circ U_2 \circ U_1$ (all reversals with respect to same density $\pi$). Then, we apply this to the corrected $T_m$ and $U_m$: $T_m$ is the composition of three steps that hold $\pi_m$ invariant. Thus, its reversal $U_m$ is given by the composition of the reversals of those steps, applied in reversed order. A full proof is in Appendix \ref{app:proofrev}.
\end{proof}

We now define the "uncorrected" transitions used by \UHA, shown in Algs.~\ref{alg:uncorrectedtm} and \ref{alg:uncorrectedum}. These are just the transitions used by Hamiltonian AIS but without the accept-reject steps. (If Hamiltonian dynamics are simulated exactly, the acceptance rate is one and the uncorrected and corrected transitions are equivalent.) We emphasize that, for the "uncorrected" transitions, $T_m$ does not exactly hold $\pi_m$ invariant and $U_m$ is not the reversal of $T_m$. Thus, their ratio does not give a simple expression in terms of $\unnorm \pi_m$ as in eq.~\ref{eq:haisest}. Nevertheless, the following result shows that their ratio has a simple form.
\begin{algorithm}[ht]
\caption{Uncorrected $T_m(z_{m+1}, \rho_{m+1} \vert z_m, \rho_m )$}
\label{alg:uncorrectedtm}
\begin{algorithmic}
\State 1. Sample $\rho'_m$ from some $s(\rho'_m \vert \rho_m)$ that leaves $\MD(\rho)$ invariant. Set $z'_m \leftarrow z_m$.
\State 2. Simulate Hamiltonian dynamics as \small$(z''_m, \rho''_m) \leftarrow \T_m(z'_m, \rho'_m)$\normalsize.
\State 3. Reverse the momentum as \small$(z_{m+1}, \rho_{m+1}) \leftarrow (z''_m, -\rho''_m)$\normalsize.
\State \Return $(z_{m+1}, \rho_{m+1})$
\end{algorithmic}
\end{algorithm}

\begin{algorithm}[ht]
\caption{Uncorrected $U_m(z_m, \rho_m \vert z_{m+1}, \rho_{m+1})$}
\label{alg:uncorrectedum}
\begin{algorithmic}
\State 1. Set \small$(z''_m, \rho''_m) \leftarrow (z_{m+1}, -\rho_{m+1})$\normalsize.
\State 2. Simulate Hamiltonian dynamics as \small$(z'_m, \rho'_m) \leftarrow \T_m(z''_m, \rho''_m)$\normalsize.
% \State 3. Sample $\rho_m$ from some $s(\cdot\vert \rho'_m)$ that leaves $\MD(\rho)$ invariant. Set $z_m \leftarrow z'_m$.
\State 3. Sample $\rho_m$ from $s_\mathrm{rev}(\rho_m\vert \rho'_m)$, the reversal of $s(\rho'_m\vert \rho_m)$ with respect to $\MD(\rho_m)$. Set $z_m \leftarrow z'_m$.
\State \Return $(z_{m}, \rho_{m})$
\end{algorithmic}
\end{algorithm}

\begin{theorem} \label{thm:ratioUT}
Let $T_m$ and $U_m$ be the uncorrected transitions defined in Algs.~\ref{alg:uncorrectedtm} and \ref{alg:uncorrectedum}, and let the dynamics simulator $\T_m(z, \rho)$ be volume preserving and self inverting. Then, 
\begin{equation}
\frac{U_m(z_m, \rho_m \vert z_{m+1}, \rho_{m+1})}{T_m(z_{m+1}, \rho_{m+1} \vert z_m, \rho_m)} = \frac{\MD(\rho_m)}{\MD(\rho'_m)}, \label{eq:R2thm1}
\end{equation}
where $\rho'_{m}$ is the second component of $\T_m(z_{m+1},-\rho_{m+1})$. (That is, $\rho'_{m}$ from Algs.~\ref{alg:uncorrectedtm} and \ref{alg:uncorrectedum}.)
\end{theorem}
\vspace{-0.4cm}
% \begin{proof}[(Proof Sketch.)]
% We derive the densities for $T_m$ and $U_m$ using the rule for transformation of densities under invertible mappings, using that $\mathcal{T}_m$ is self-inverting and volume preserving. We represent delta distributions as Gaussian in the zero variance limit. Taking the ratio gives the result. A full proof is in Appendix \ref{app:proofthm}.
% \end{proof}
\begin{proof}[(Proof Sketch.)]
We consider variants of Algs.~\ref{alg:uncorrectedtm} and \ref{alg:uncorrectedum} in which each time $z$ is assigned we add Gaussian noise with some variance $aI$. We then derive the densities for $T_m$ and $U_m$ using the rule for transformation of densities under invertible mappings, using that $\mathcal{T}_m$ is self-inverting and volume preserving. Taking the ratio gives eq.~\ref{eq:R2thm1}. Since this is true for arbitrary $a$, we take the stated result as the limit as $a \to 0$. A full proof is in Appendix \ref{app:proofthm}.
\end{proof}

\vspace{-0.2cm}
As an immediately corollary of eq. \ref{eq:haisobv} and Theorem \ref{thm:ratioUT} we get that for \UHA
\begin{equation}
\frac{\unnorm p(z_{1:M}, \rho_{1:M})}{q(z_{1:M}, \rho_{1:M})} = \frac{\unnorm p(z_M)}{q(z_1)} \prod_{m = 1}^{M - 1} \frac{\MD(\rho_{m+1})}{\MD(\rho'_{m})}. \label{eq:estuhais1}
\end{equation}
This ratio can be used to get an expression for the lower bound $\ELBO(q(z_{1:M}, \rho_{1:M}), \unnorm p(z_{1:M}, \rho_{1:M}))$ obtained with \UHA. As mentioned in Section \ref{sec:preliminaries}, the parameters of the augmented distributions are tuned to maximize the ELBO, equivalent to minimizing the KL-divergence from $q$ to $\unnorm p$. While computing this ELBO exactly is typically intractable, an unbiased estimate can be obtained using a sample from $q(z_{1:M}, \rho_{1:M})$ as shown in Alg.~\ref{alg:sampleuhais}. If sampling is done using reparameterization, then unbiased reparameterization gradients may be used together with stochastic optimization algorithms to optimize the lower bound. In contrast, the variational lower bound obtained with Hamiltonian AIS (see Alg.~\ref{alg:sampleais} in Appendix \ref{app:HAISboundalg}) does not allow the computation of unbiased reparameterization gradients.

\begin{algorithm}[ht]
\caption{Generating the (differentiable) uncorrected Hamiltonian annealing variational bound.} %\jd{This name is a bit modest. We should have a name like: OUR AMAZING MCMC-VI ALGORITHM. And I'd like to explain somewhere here what this algorithm does (gives an estimate of the ELBO nearly as good as AIS but differentiable)}}
\label{alg:sampleuhais}
\begin{algorithmic}
\State Sample $z_1 \sim q$ and $\rho_1 \sim \MD$.
\State Initialize estimator as $\mathcal{L} \leftarrow -\log q(z_1)$.
\For{$m = 1, 2, \cdots , M-1$}
	\State Run uncorrected $T_m$ (Alg.~\ref{alg:uncorrectedtm}) on input $(z_m, \rho_m)$, storing $\rho'_m$ and the output $(z_{m+1}, \rho_{m+1})$.
	\State Update estimator as $\mathcal{L} \leftarrow \mathcal{L} + \log \left( \MD(\rho_{m+1}) / \MD(\rho'_m) \right)$.
\EndFor
\State Update estimator as $\mathcal{L} \leftarrow \mathcal{L} + \log \unnorm p(z_M)$.
\State \Return $R$
\end{algorithmic}
\end{algorithm}

\vspace{-0.3cm}
\subsection{Algorithm Details}

\textbf{Simulation of dynamics.} We use the leapfrog operator with step-size $\epsilon$ to simulate Hamiltonian dynamics. This has unit Jacobian and satisfies $\mathcal{T}_m = \mathcal{T}_m^{-1}$ (if the momentum is negated after the simulation), which are the properties required for eq.~\ref{eq:estuhais1} to be correct (see Theorem \ref{thm:ratioUT}).

\textbf{Momentum distribution and resampling.} We set the momentum distribution $\MD(\rho) = \mathcal{N}(\rho \vert 0, \Sigma)$ to be a Gaussian with mean zero and covariance $\momcov$. The resampling distribution $s(\rho'\vert \rho)$ must hold this distribution invariant. As is common we use $s(\rho' \vert \rho) = \mathcal{N}(\rho' \vert \eta \rho, (1-\eta^2) \Sigma)$, where $\eta \in [0, 1)$ is the damping coefficient. If $\eta = 0$, the momentum is completely replaced with a new sample from $\MD$ in each iteration (used by HMC and Langevin dynamics \cite{neal2011mcmc, welling2011bayesian}). For larger $\eta$, the momentum becomes correlated between iterations, which may help suppress random walk behavior and encourage faster mixing \cite{cheng2018underdamped} (used by the underdamped variants of HMC and Langevin dynamics \cite{neal2011mcmc}).

\textbf{Bridging densities.} We set $\unnorm \pi_m(z, \rho) = q(z, \rho)^{1 - \beta_m} \unnorm p(z, \rho)^{\beta_m}$, where $\beta_m \in[0, 1]$ and $\beta_m < \beta_{m+1}$.

\textbf{Computing gradients.} We set the initial distribution $q(z_1)$ to be a Gaussian, and perform all sampling operations in Alg.~\ref{alg:sampleuhais} using reparameterization \cite{vaes_welling, rezende2014stochastic, doublystochastic_titsias}. Thus, the whole procedure is differentiable and reparameterization-based gradients may be used to tune parameters by maximizing the ELBO. These parameters include the initial distribution $q(z_1)$, the covariance $\momcov$ of the momentum distribution, the step-size $\epsilon$ of the integrator, the damping coefficient $\eta$ of the momentum resampling distribution, and the parameters of the bridging densities (including $\beta$), among others. As observed in Section \ref{sec:tuningmore} tuning all of these parameters may lead to considerable performance improvements.

%!TEX root = ../main.tex

\section{Related Work} \label{sec:relwork}

\UHA\ and slight variations have been proposed in concurrent work by Thin et al. \cite{thin2021monte}, who use uncorrected Langevin dynamics together with the uncorrected reversal to build variational lower bounds, and by Zhang et al. \cite{zhang2021differentiable}, who proposed \UHA\ with under-damped Langevin dynamics together with a convergence analysis for linear regression models.

There are three other lines of work that produce differentiable variational bounds integrating Monte Carlo methods. One is Hamiltonian VI (HVI) \cite{salimans2015markov, wolf2016variational}. It uses eq.~\ref{eq:haisobv} to build a lower bound on $\log Z$, with $T_m$ set to an uncorrected Hamiltonian transition (like \UHA\ but without bridging densities) and $U_m$ set to conditional Gaussians parameterized by learnable functions. Typically, a single transition is used, and the parameters of the transitions are learned by maximizing the resulting ELBO.\footnote{The formulation of HVI allows the use of more than one transition. However, this leads to an increased number of reverse models that must be learned, and thus not typically used in practice. Indeed, experiments by Salimans et al. \cite{salimans2015markov} use only one HMC step while varying the number of leapfrog integration steps, and results from Wolf et al. \cite{wolf2016variational} show that increasing the number of transitions may actually yield worse bounds (they conjecture that this is due to the difficulty of learning inverse dynamics.).}

A second method is given by Hamiltonian VAE (HVAE) \cite{caterini2018hamiltonian}, based on Hamiltonian Importance sampling \cite{neal2005hamiltonianis}. They augment the variational distribution with momentum variables, and use the leapfrog integrator to simulate Hamiltonian dynamics (a deterministic invertible transformation with unit Jacobian) with a tempering scheme as a target-informed flow \cite{rezende2015variational, tabak2013family}.

The third method is Importance Weighting (IW) \cite{IWVAE, cremer2017reinterpreting, domke2018importance}. Here, the idea is that $\ELBO(q(z), \unnorm p(z)) \leq \mathbb{E} \log \frac{1}{K} \sum_k \unnorm p(z_k) / q(z_k)$, and that the latter bound can be optimized, rather than the traditional ELBO. More generally, other Monte-Carlo estimators can be used \cite{domke2019divide}.

Some work defines novel contrastive-divergence-like objectives in terms of the final iteration of an MCMC chain \cite{ruiz2019contrastive, li2017approximate}. These do not provide an ELBO-like variational bound. While in some cases the initial distribution can be optimized to minimize the objective \cite{ruiz2019contrastive}, gradients do not flow through the MCMC chains, meaning MCMC parameters cannot be optimized by gradient methods.

For latent variable models, Hoffman \cite{hoffman2017MCMCvae} suggested to run a few MCMC steps after sampling from the variational distribution before computing gradients with respect to the model parameters, which is expected to "debias" the gradient estimator to be closer to the true likelihood gradient. The variational distribution is simultaneously learned to optimize a standard ELBO. (AIS can also be used \cite{ding2019vaeAIS}.)

%!TEX root = ../main.tex

\section{Experiments and Results} \label{sec:exps}

This section presents results using \UHA\ for Bayesian inference problems on several models of varying dimensionality and for VAE training. We compare against Hamiltonian AIS, IW, HVI and HVAE. We report the performance of each method for different values of $K$, the number of likelihood evaluations required to build the lower bound (e.g. number of samples used for IW, number of bridging densities plus one for \UHA). Note that, for a fixed $K$, all methods have the same oracle complexity (i.e. number of target/target's gradient evaluation), and that for $K=1$ they all reduce to plain VI.
% \footnote{For IW $K$ is the number of samples used to build the lower bound, for \UHA\ and Hamiltonian AIS it is the number of bridging densities used plus one, for HVI and HVAE it is the number of leapfrog steps used plus one.}

For \UHA\ and Hamiltonian AIS we use under-damped Langevin dynamics, that is, we perform just one leapfrog step per transition and partially resample momentum. We implement all algorithms using Jax \cite{jax2018github}.

\subsection{Toy example}

This section compares results obtained with \UHA\ and IW when the target is set to a factorized Student-t with mean zero, scale one, and three degrees of freedom. We tested three different dimensionalities: $20$, $200$ and $500$. In all cases we have $\log Z = 0$, so we can exactly analyze the tightness of the bounds obtained by the methods. We set the initial approximation to be a mean-field Gaussian, and optimize the objective using Adam \cite{adam} with a step-size of $0.001$ for $5000$ steps. For \UHA\ we tune the initial approximation $q(z)$, the integrator's step-size $\epsilon$ and the damping coefficient $\eta$.

We ran \UHA\ for $K\in\{4, 16, 64, 128\}$ and IW for $K\in\{128, 1024\}$. Table \ref{tab:toy} shows the results for the three dimensionalities considered. It can be observed that \UHA\ performs significantly better than IW as the dimensionality increases; for the target with dimension $500$, \UHA\ with $K=16$ yields better bounds than IW with $K=1024$. On the other hand, the methods perform similarly for the low dimensional target. Finally, in this case both methods have similar time costs. For instance, for $K=128$ \UHA\ takes $14.2$ seconds to optimize and IW takes $13.9$.

% \begin{table}[]
% \caption{\textbf{\UHA\ yields better bounds than plain VI and IW.} ELBO achieved by different methods when using a Student-t target distribution ($\log Z = 0$), higher is better. Optimization times (in seconds) for different methods are $2.7$ for plain VI, $[3.3, 4.2, 8.3, 14.2]$ for \UHA\ with $K \in \{4, 16, 64, 128\}$, and $[15.1, 43.3]$ for IW with $K \in \{128, 1024\}$.}
% \label{tab:toy}
% \centering
% \begin{tabular}{lllllll}
% \toprule
% \multirow{2}{*}{Plain VI} & \multicolumn{4}{c}{UHA} & \multicolumn{2}{c}{IW} \\
% \cmidrule(l{2pt}r{2pt}){2-5} \cmidrule(l{2pt}r{2pt}){6-7}
% & $K=4$   & $K=16$  & $K=64$  & $K=128$ & $K=128$  & $K=1024$ \\
% \midrule
% $-0.82$ & $-0.55$ & $-0.36$ & $-0.19$ & $-0.14$ & $-0.14$ & $-0.088$ \\
% $-8.1$ & $-5.5$  & $-3.5$  & $-1.9$  & $-1.4$  & $-3.7$ & $-2.9$ \\
% $-20.5$ & $-13.9$ & $-9.0$  & $-5.2$  & $-3.8$  & $-12.0$ & $-10.4$ \\
% \bottomrule
% \end{tabular}
% \end{table}

\begin{table}[]
\caption{\textbf{Our method (\UHA) yields better bounds than importance weighting (IW) for moderate or high dimensions.} ELBO achieved by different methods when using a Student-t target distribution of varying dimensionality, higher is better. Since the target is normalized, a perfect inference algorithm would achieve the true value of $\log Z = 0$.}
\label{tab:toy}
\centering
\begin{tabular}{llllllll}
\toprule
Target & Plain VI & \multicolumn{4}{c}{UHA} & \multicolumn{2}{c}{IW} \\
\cmidrule(l{2pt}r{2pt}){2-2} \cmidrule(l{2pt}r{2pt}){3-6} \cmidrule(l{2pt}r{2pt}){7-8}
dimension  &$K=1$& $K=4$   & $K=16$  & $K=64$  & $K=128$ & $K=128$  & $K=1024$ \\
\midrule
$20$ & $-0.82$ & $-0.55$ & $-0.36$ & $-0.19$ & $-0.14$ & $-0.14$ & $-0.088$ \\
$200$ & $-8.1$ & $-5.5$  & $-3.5$  & $-1.9$  & $-1.4$  & $-3.7$ & $-2.9$ \\
$500$ & $-20.5$ & $-13.9$ & $-9.0$  & $-5.2$  & $-3.8$  & $-12.0$ & $-10.4$ \\
\bottomrule
\end{tabular}
\end{table}

\subsection{Inference tasks} \label{sec:exps_real}

This section shows results using \UHA\ for Bayesian inference tasks. For this set of experiments, for \UHA\ we tune the initial distribution $q(z)$, the integrator's step-size $\epsilon$ and the damping coefficient $\eta$. We include detailed results tuning more parameters in Section \ref{sec:tuningmore}.

\textbf{Models.} We consider four models: \textit{Brownian motion} ($d = 32$), which models a Brownian Motion process with a Gaussian observation model; \textit{Convection Lorenz bridge} ($d = 90$), which models a nonlinear dynamical system for atmospheric convection; and \textit{Logistic regression} with the a1a ($d = 120$) and madelon ($d = 500$) datasets. The first two obtained from the ``Inference gym'' \cite{inferencegym2020}.

\textbf{Baselines.} We compare \UHA\ against IW, HVAE, a simple variant of HVI, and Hamiltonian AIS (HAIS). For all methods which rely on HMC (i.e. all except IW) we use a singe integration step-size $\epsilon$ common to all dimensions and fix the momentum distribution to a standard Gaussian. For HVI we learn the initial distribution $q(z)$, integration step-size $\epsilon$ and the reverse dynamics $U_m$ (set to a factorized Gaussian with mean and variance given by affine functions), and for HVAE we learn $q(z)$, $\epsilon$ and the tempering scheme (we use the quadratic scheme parameterized by a single parameter).

\textbf{Training details.} We set $q(z)$ to be a mean-field Gaussian initialized to a maximizer of the ELBO, and tune the parameters of each method by running Adam for $5000$ steps. We repeat all simulations for different step-sizes in $\{10^{-3}, 10^{-4}, 10^{-5}\}$, and select the best one for each method. Since Hamiltonian AIS' parameters cannot be tuned by gradient descent, we find a good pair $(\epsilon, \eta)$ by grid search. We consider $\eta \in \{0.5, 0.9, 0.99\}$ and three values of $\epsilon$ that correspond to three different rejection rates: $0.05, 0.25$ and $0.5$. We tested all 9 possible combinations and selected the best one.

Results are shown in Fig.~\ref{fig:comparison1}.
% To simplify comparisons against IW, results are shown as a function of $K$, the number of likelihood evaluations required by each method to build the lower bound.\footnote{For IW $K$ is the number of samples used to build the lower bound, for \UHA\ and Hamiltonian AIS it is the number of bridging densities plus one, for HVI and HVAE it is the number of leapfrog steps used plus one.}
Our method yields better lower bounds than all other competing approaches for all models considered, and that increasing the number of bridging densities consistently leads to better results. The next best performing method is Hamiltonian AIS. IW also shows a good performance for the lower dimensional model \textit{Brownian motion}. However, for models of higher dimensionality IW leads to bounds that are several nats worse than the ones achieved by \UHA. Finally, HVI and HVAE yield bounds that are much worse than those achieved by the other three methods, and do not appear to improve consistently for larger $K$. For HVAE, these results are consistent with the ones in the original paper \cite[\S4]{caterini2018hamiltonian}, in that higher $K$ may sometimes hurt performance. For HVI, we believe this is related to the use of just one HMC step and suboptimal inverse dynamics.

Optimization times for Plain VI, IW and \UHA\ (the latter two with $K=32$) are $2.4, 3.4 \mbox{ and } 4.4$ seconds for the \textit{Brownian motion} dataset, $2.5, 6.8 \mbox{ and } 6.9$ seconds for \textit{Lorenz convection}, $2.8, 8.3 \mbox{ and } 19.9$ seconds for \textit{Logistic regression (A1A)}, and $4.6, 16.6 \mbox{ and } 121.2$ seconds for \textit{Logistic regression (Madelon)}. While IW and \UHA\ have the same oracle complexity for the same $K$, we see that the difference between their time cost depends on the specific model under consideration. All other methods that use HMC have essentially the same time cost as \UHA.

\begin{figure}[ht]
  \centering
  \includegraphics[scale=0.29, trim = {0 2.1cm 0 0}, clip]{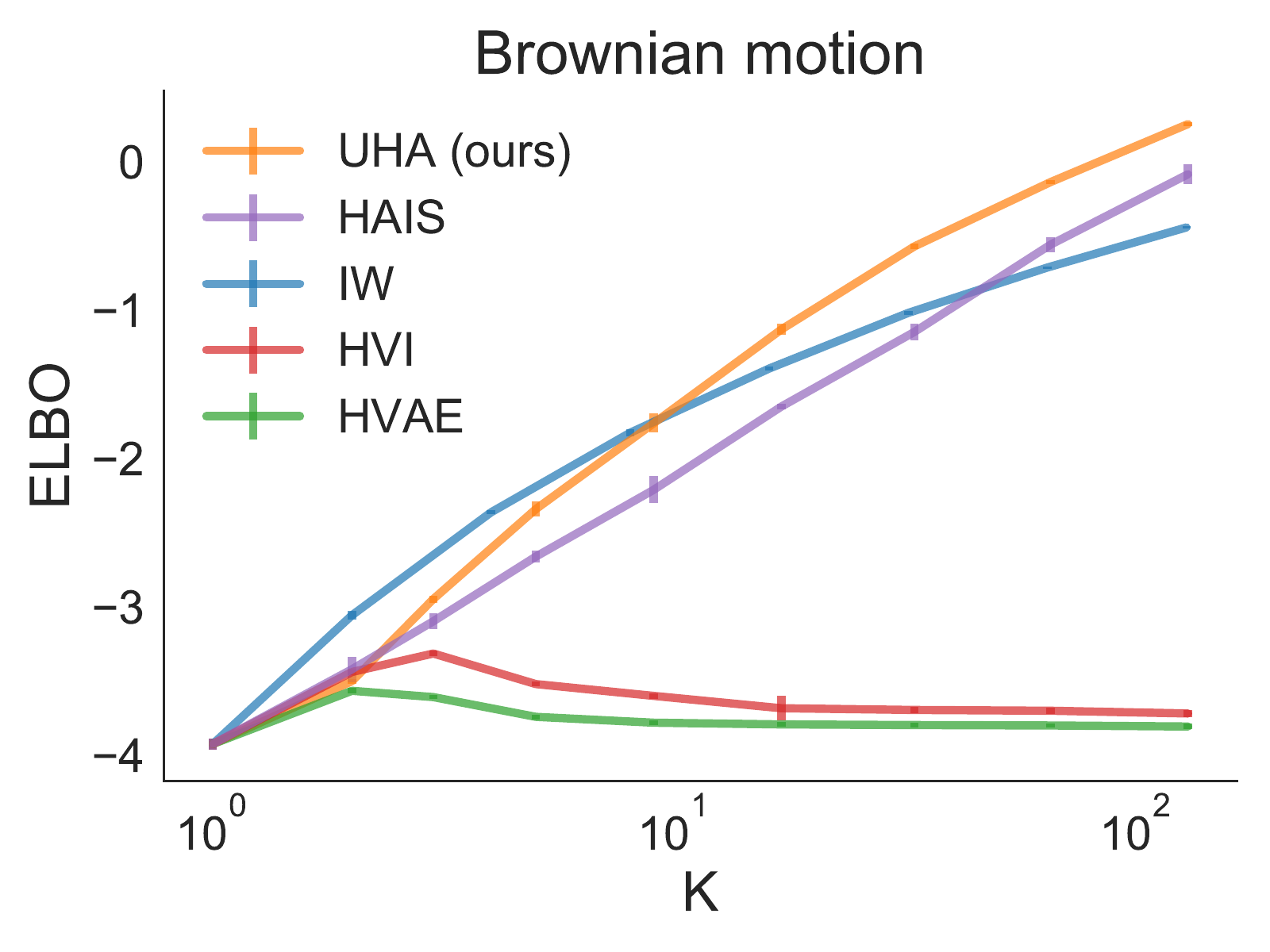}
  \includegraphics[scale=0.29, trim = {1.3cm 2.1cm 0 0}, clip]{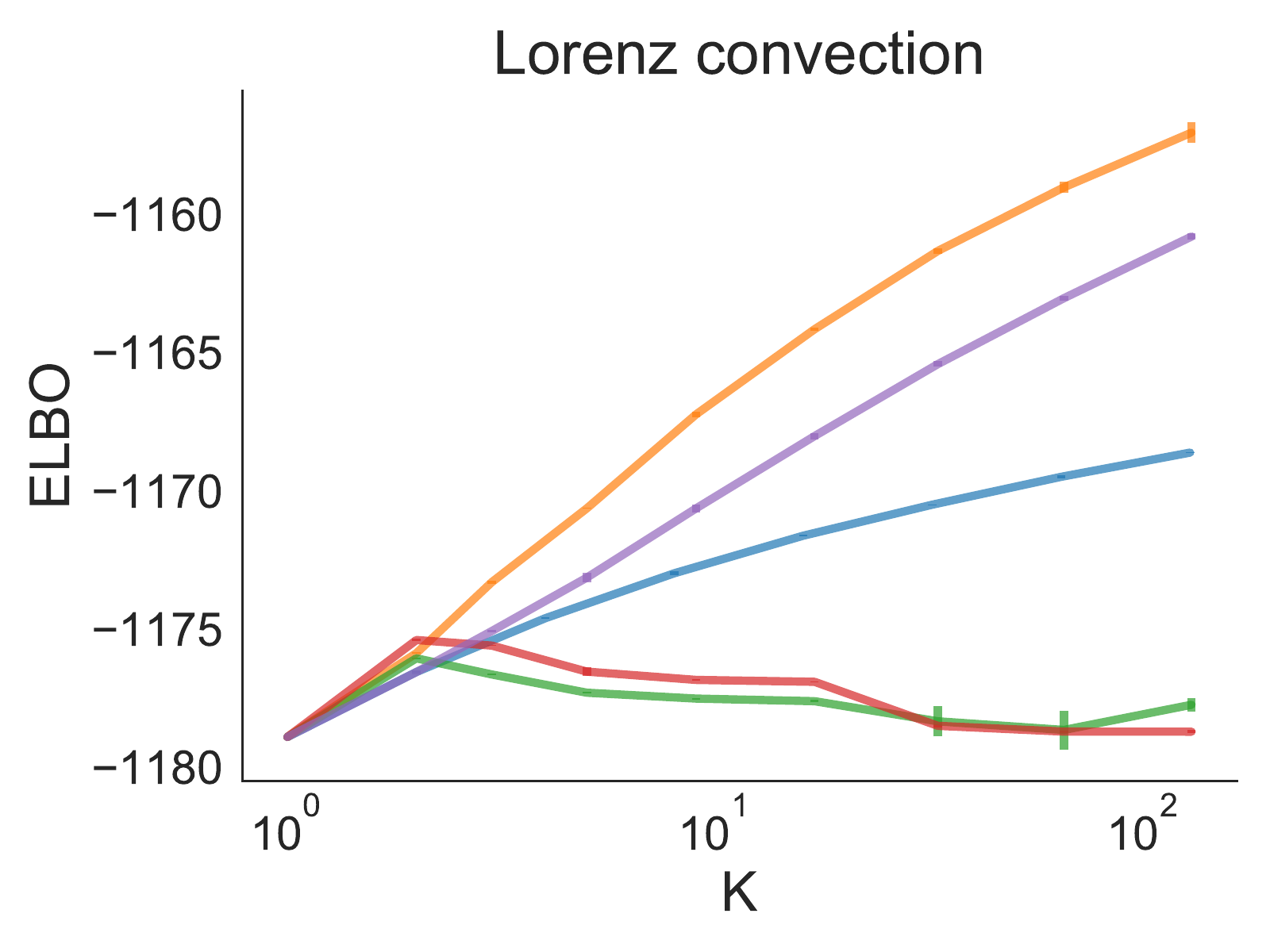}

  \includegraphics[scale=0.29, trim = {0 0 0 0}, clip]{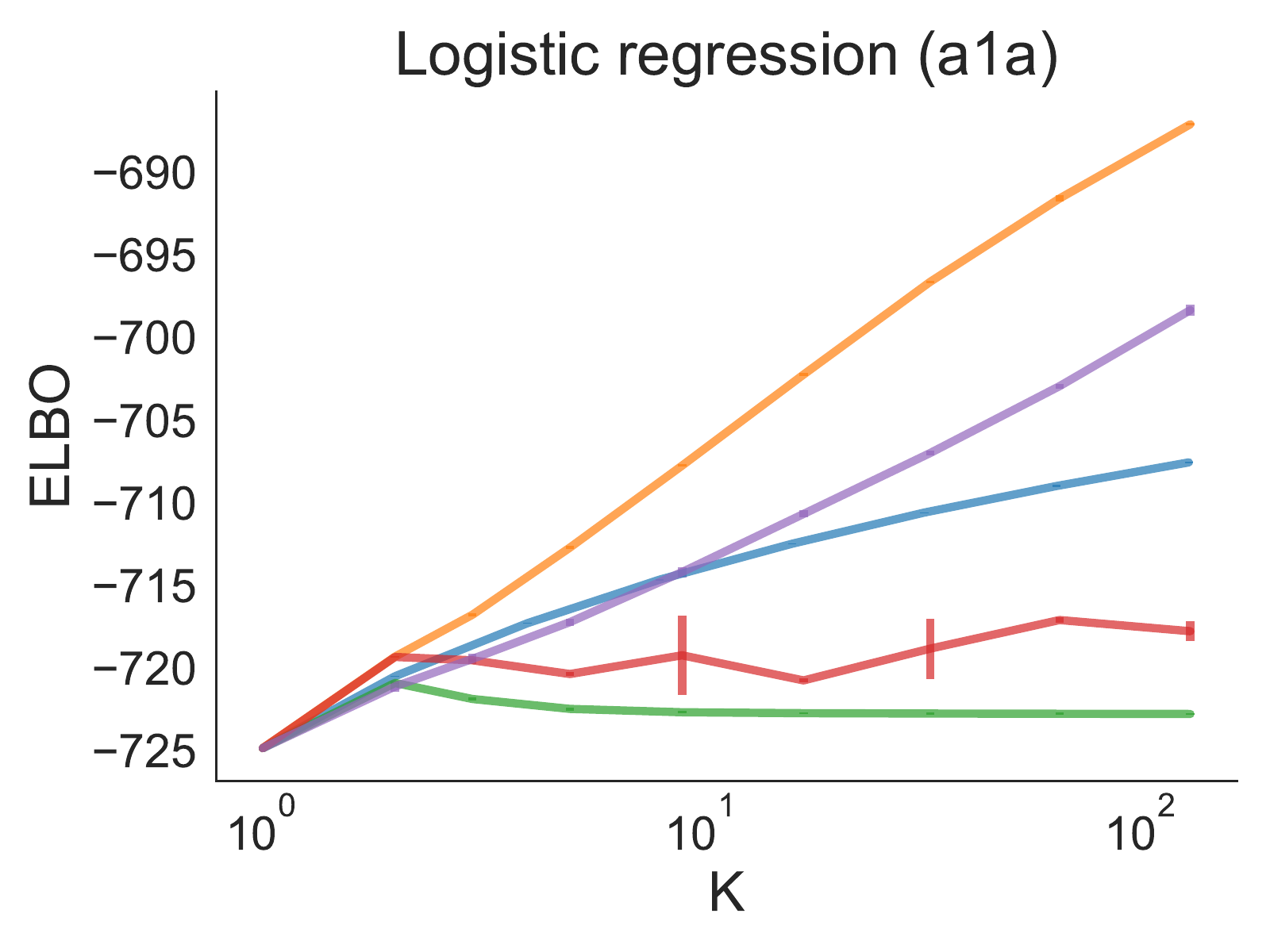}
  \includegraphics[scale=0.29, trim = {1.3cm 0 0 0}, clip]{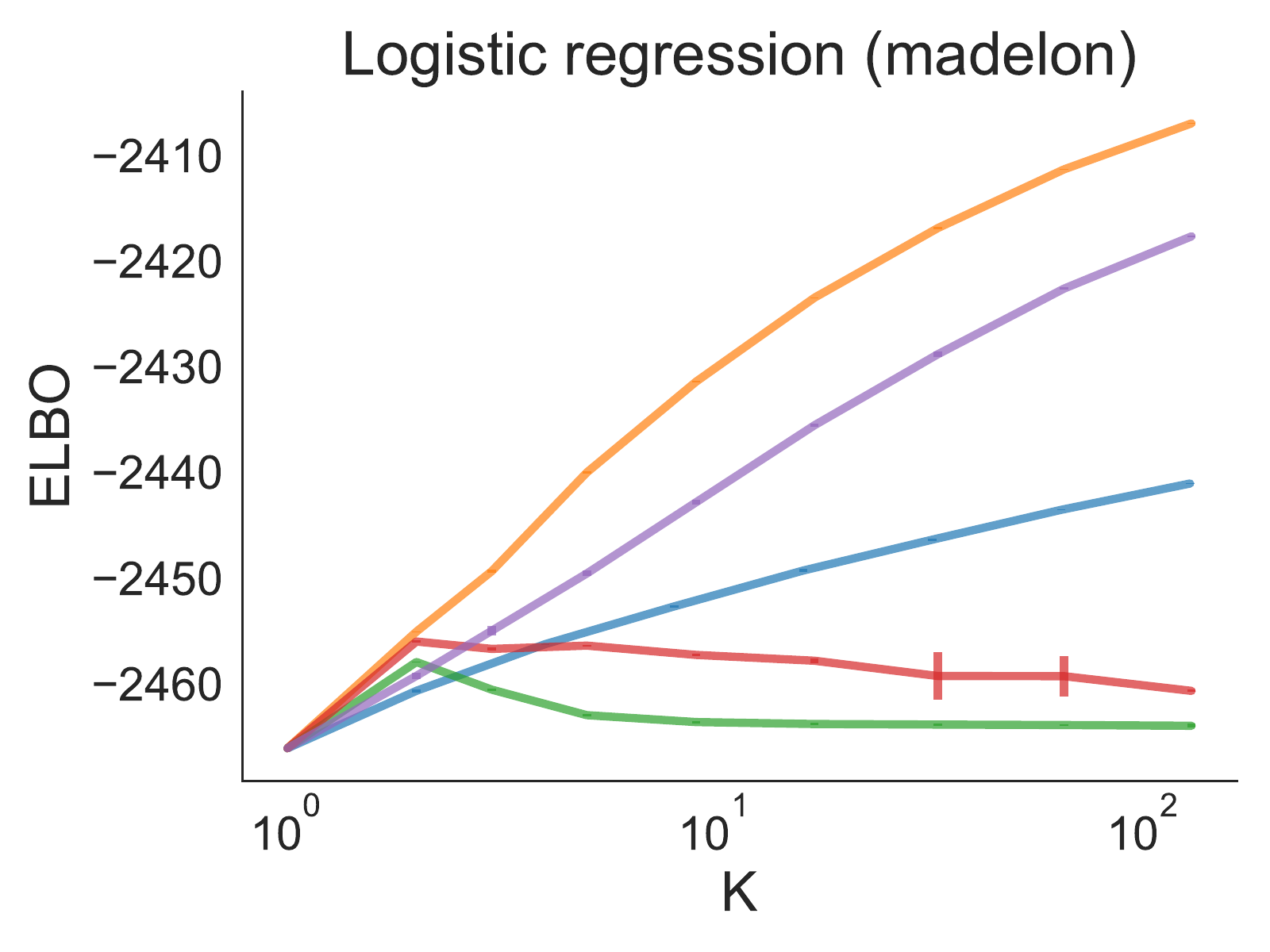}
  \caption{\textbf{Our method achieves much better bounds than other competing methods.} $K$ represents the number of likelihood evaluations to build the lower bound required by each method. The leftmost point of all lines coincide because, for $K = 1$, all methods reduce to plain VI. Vertical bars indicate one standard deviation obtained by running simulations with four different random seeds. %Optimization times for different methods (in seconds, for the datasets \textit{Brownian motion}, \textit{Lorenz convection}, \textit{Logistic reg (A1A)}, \textit{Logistic reg (Madelon)} in this order) are $[2.4, 2.7, 2.9, 4.5]$ for plain VI, $[4.5, 6.9, 22.7, 143.2]$ for \UHA\ with $K=32$, and $[9.1, 3.4, 9.2, 22.3]$ for IW with $K=32$.
  }
  \label{fig:comparison1}
\end{figure}

\subsubsection{Tuning More Parameters with \UHA} \label{sec:tuningmore}

A basic version of \UHA\ involves fitting a variational distribution using plain VI, and then tuning the integration step-size $\epsilon$ and the damping coefficient $\eta$. However, more parameters could be tuned:

\vspace{-0.2cm}
\begin{itemize}[leftmargin=0.5cm]
\setlength\itemsep{-0.06cm}
\item Moment distribution cov $\momcov$: We propose to learn a diagonal matrix instead of using the identity.
\item Bridging densities' coefficients $\beta_m$: Typically $\beta_m = m / M$. We propose to learn the sequence $\beta$, with the restrictions $\beta_0 = 0$, $\beta_M = 1$, $\beta_m < \beta_{m+1}$ and $\beta_m \in [0, 1]$.
\item Initial distribution $q(z)$: Instead of fixing $q(z)$ to be a maximizer of the typical ELBO, we propose to learn it to maximize the augmented ELBO obtained using \UHA.
\item Integrator's step-size $\epsilon$: Instead of learning a unique step-size $\epsilon$, we propose to learn a step-size that is a function of $\beta$, i.e. $\epsilon(\beta)$. In our experiments we use an affine function.
\item Bridging densities parameters $\psi$: Instead of setting the $m$-th bridging density to be $q^{1 - \beta_m} p^{\beta_m}$, we propose to set it to $q_{\psi(\beta_m)}^{1 - \beta_m} \, p^{\beta_m}$, where $q_{\psi(\beta_m)}$ is a mean-field Gaussian with a mean and diagonal covariance specified as affine functions of $\beta$.
\end{itemize}

We consider the four models described previously and compare three methods: \UHA\ tuning all parameters described above, \UHA\ tuning only the pair $(\epsilon, \eta)$, and Hamiltonian AIS with parameters $(\epsilon, \eta)$ obtained by grid-search. We perform the comparison for $K$ ranging from $2$ to $512$. (For $K \geq 64$ we tune the \UHA's parameters using $K = 64$ and extrapolate them as explained in Appendix \ref{sec:extrapol}.)

Results are shown in Fig.~\ref{fig:tuning_more_1}. It can be observed that tuning all parameters with \UHA\ leads to significantly better lower bounds than those obtained by Hamiltonian AIS (or \UHA\ tuning only $\epsilon$ and $\eta$). Indeed, for the Logistic regression models, \UHA\ tuning all parameters for $K = 64$ leads to results comparable to the ones obtained by Hamiltonian AIS with $K = 512$.

\begin{figure}[ht]
  \centering
  \includegraphics[scale=0.2, trim = {0 0 0 0}, clip]{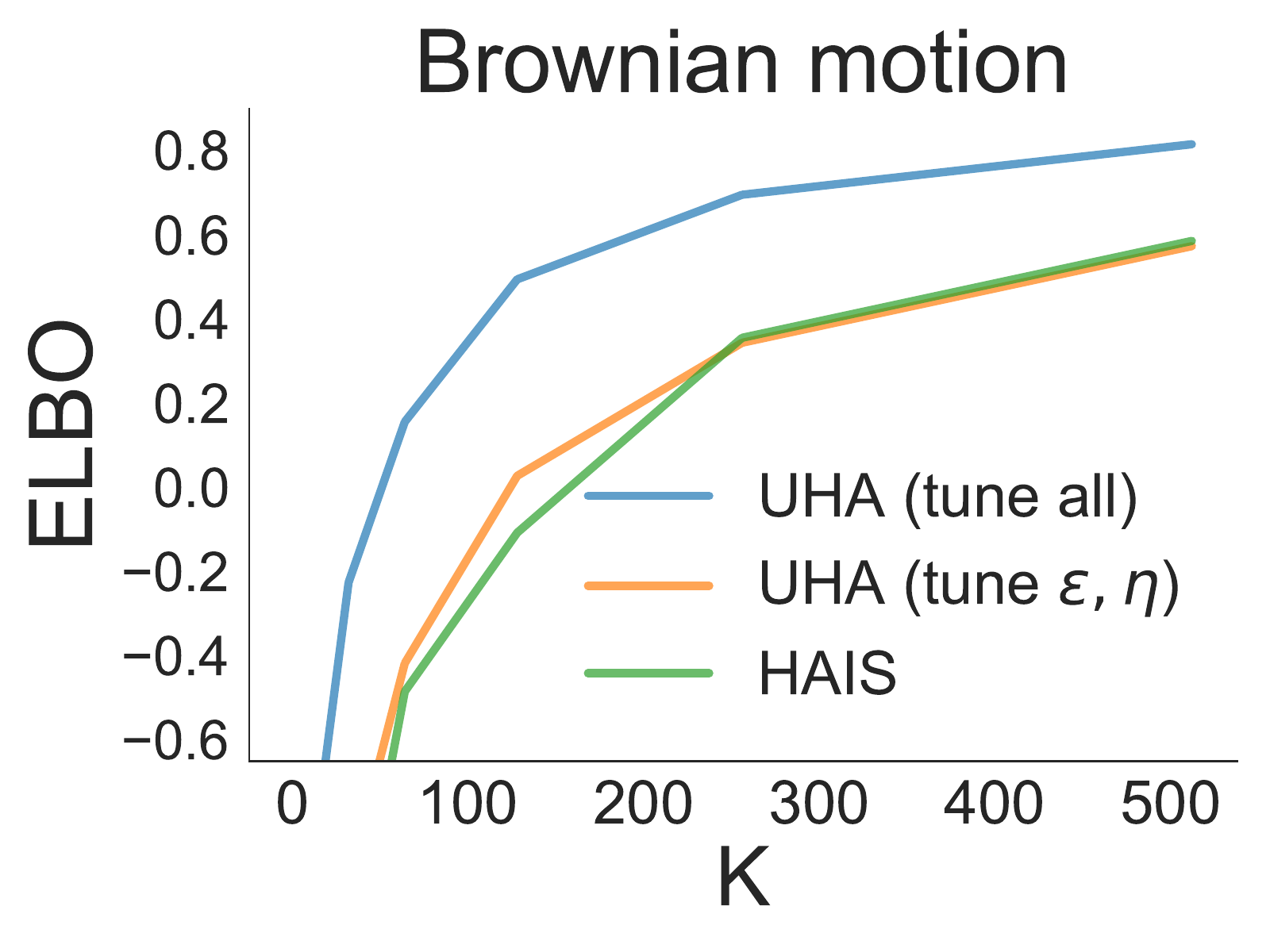}\hfill
  \includegraphics[scale=0.2, trim = {1.5cm 0 0 0}, clip]{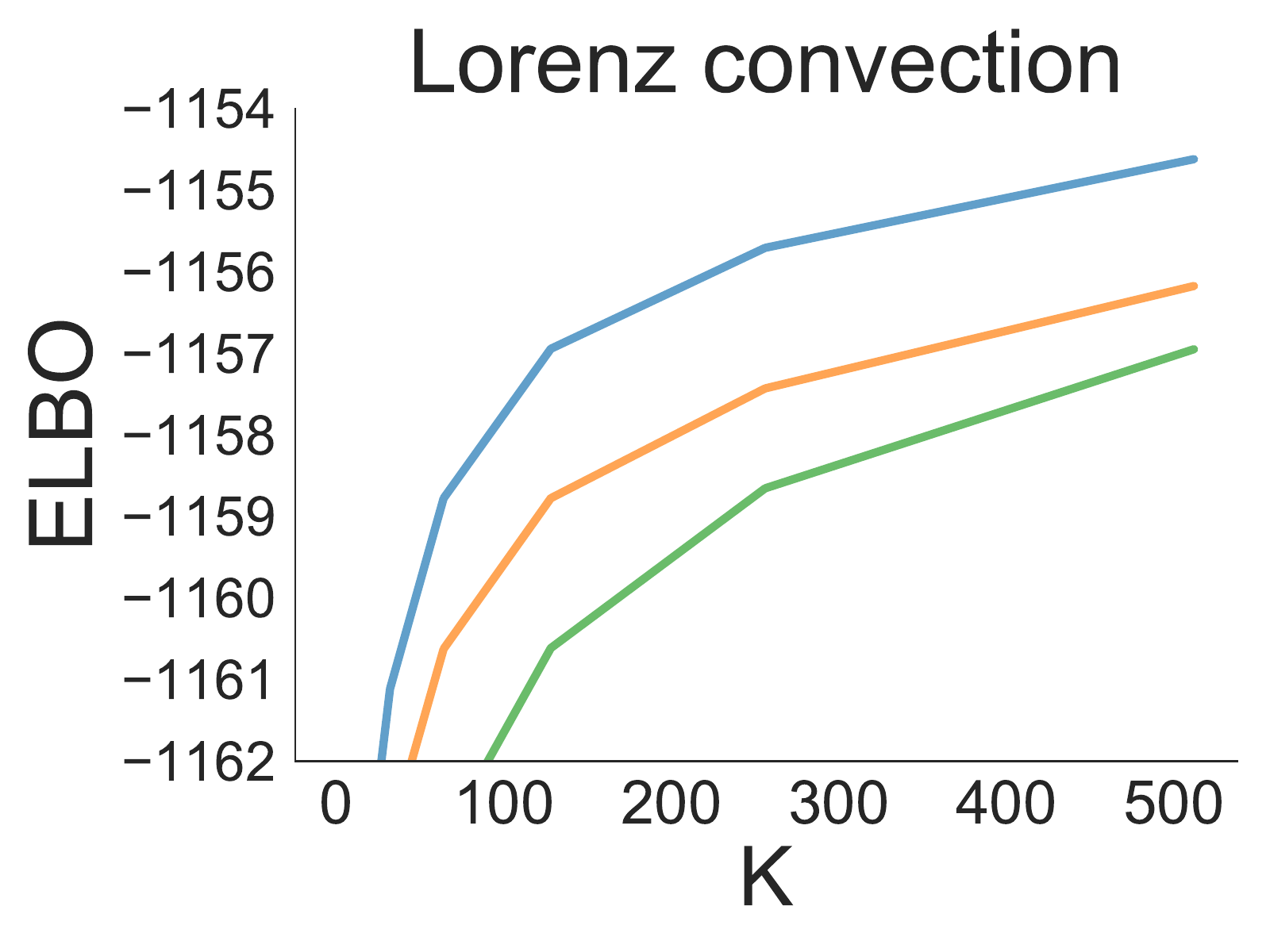}\hfill
  \includegraphics[scale=0.2, trim = {1.5cm 0 0 0}, clip]{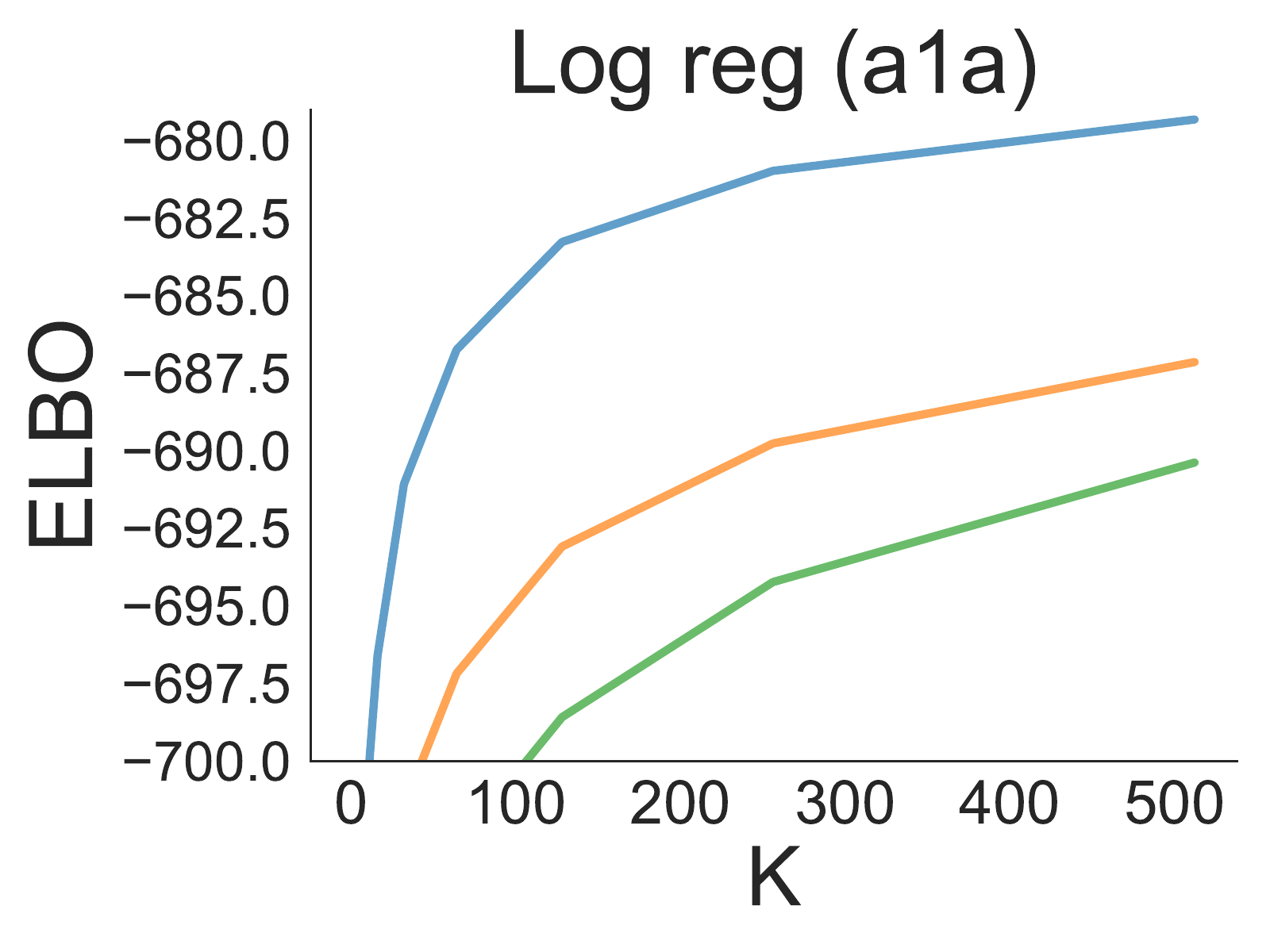}\hfill
  \includegraphics[scale=0.2, trim = {1.5cm 0 0 0}, clip]{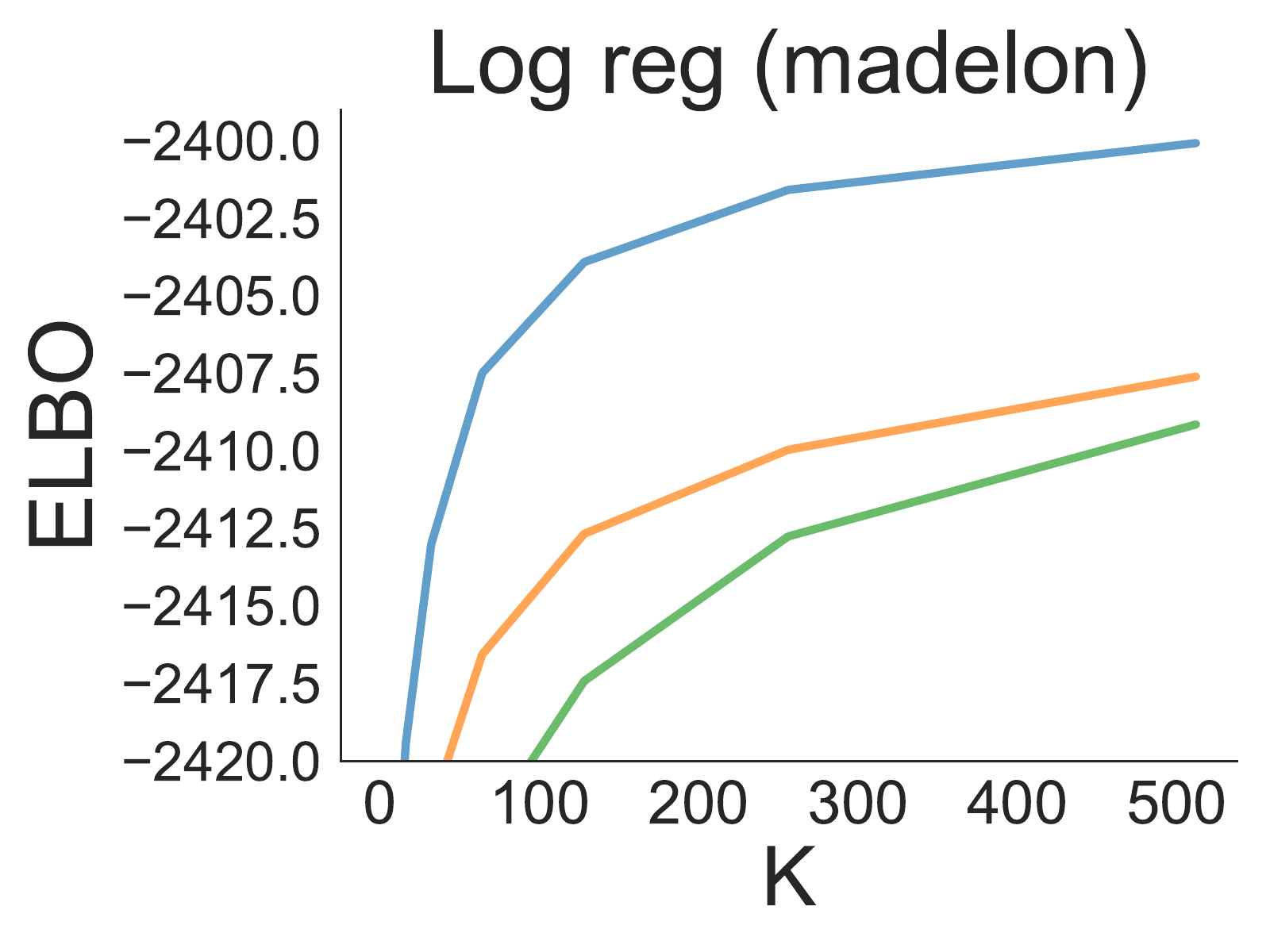}
  \caption{\textbf{\UHA\ tuning all parameters leads to better performance than other methods.}}
  \label{fig:tuning_more_1}
\end{figure}

To verify what parameters lead to larger performance improvements, we tested \UHA\ with $K = 64$ tuning different subsets of $\{\epsilon, \eta, \momcov, \beta, q(z), \epsilon(\beta), \psi(\beta)\}$. Fig.~\ref{fig:tuning_more_1_fig} shows the results. It can be observed that tuning the bridging parameters $\beta$ and the initial approximation $q(z)$ leads to the largest gains in performance, and that tuning all parameters always outperforms tuning smaller subsets of parameters. We show a more thorough analysis, including more subsets and values of $K$ in Appendix \ref{app:tunemoreuha}.

\begin{figure}[ht]
  \centering
  \includegraphics[scale=0.21, trim = {0 0 0 0}, clip]{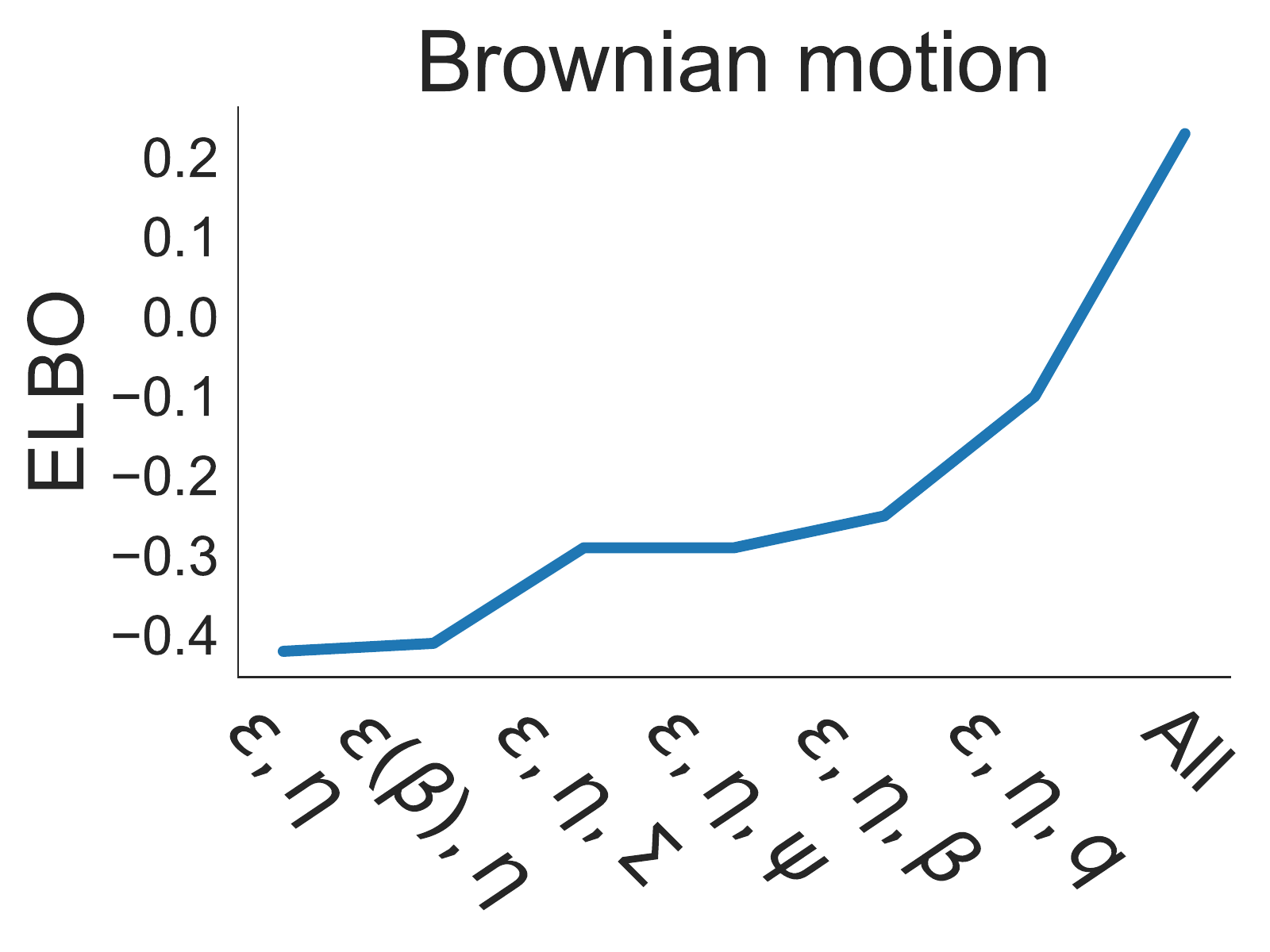}\hfill
  \includegraphics[scale=0.21, trim = {0 0 0 0}, clip]{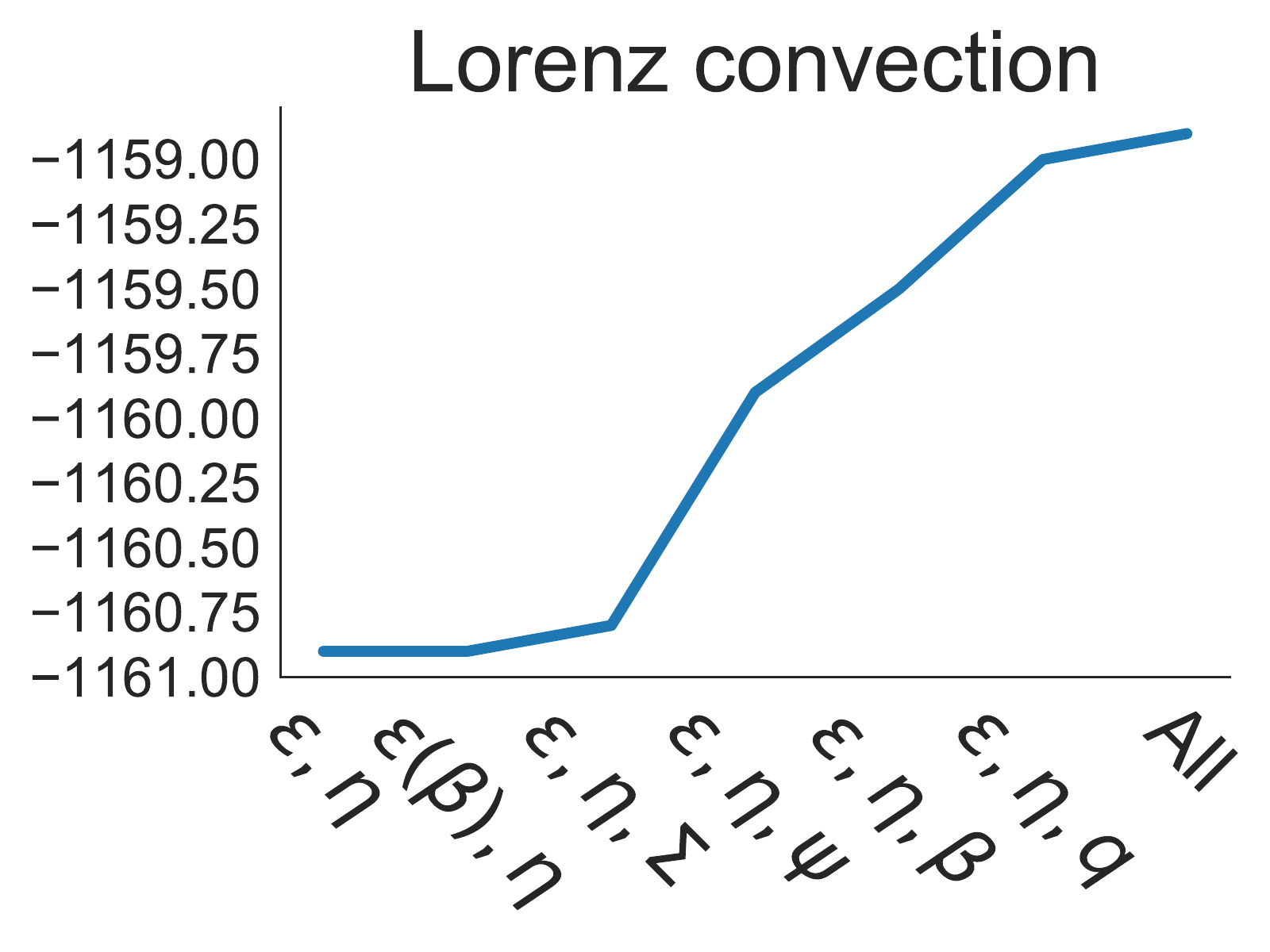}\hfill
  \includegraphics[scale=0.21, trim = {0 0 0 0}, clip]{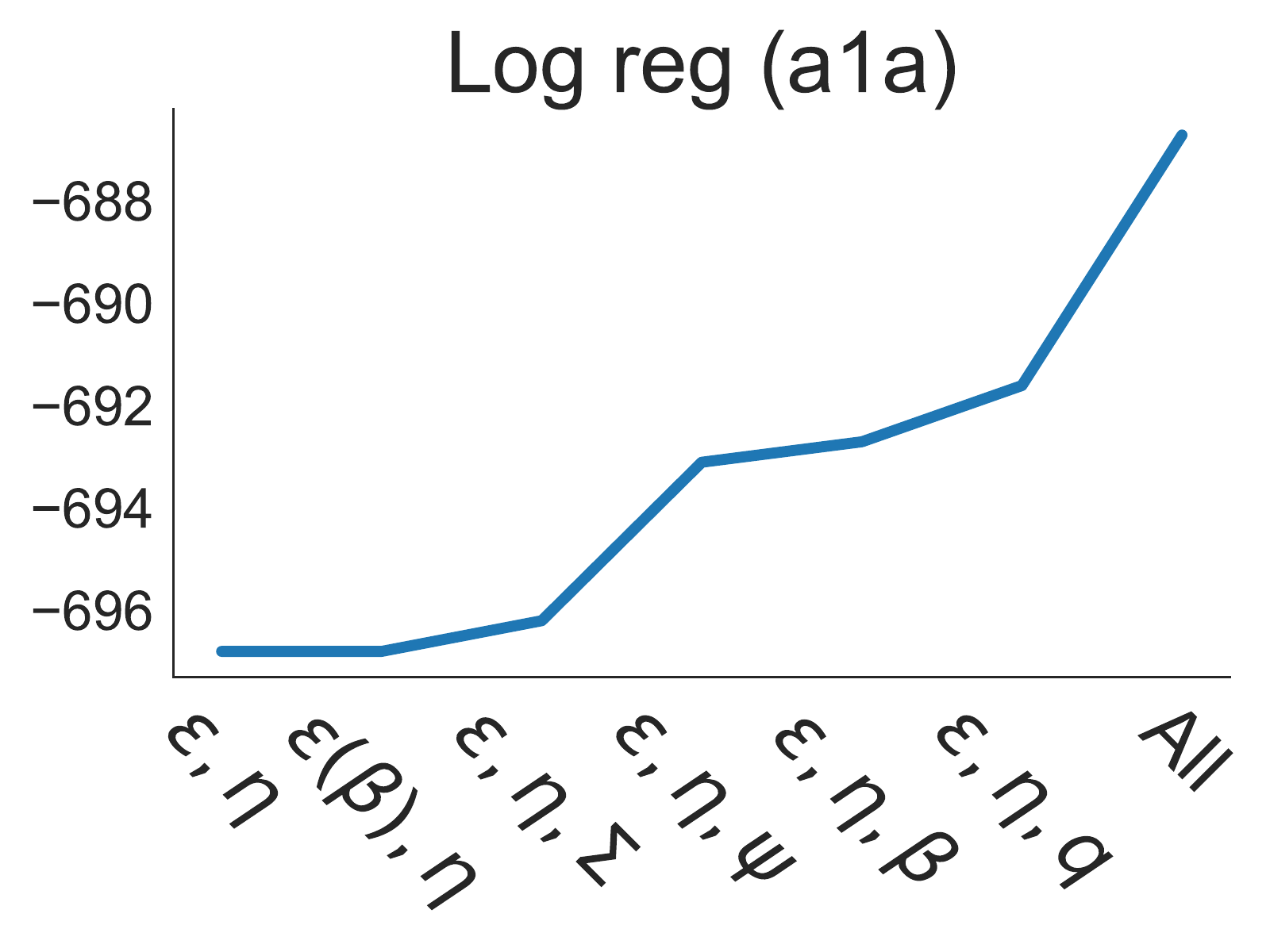}\hfill
  \includegraphics[scale=0.21, trim = {0 0 0 0}, clip]{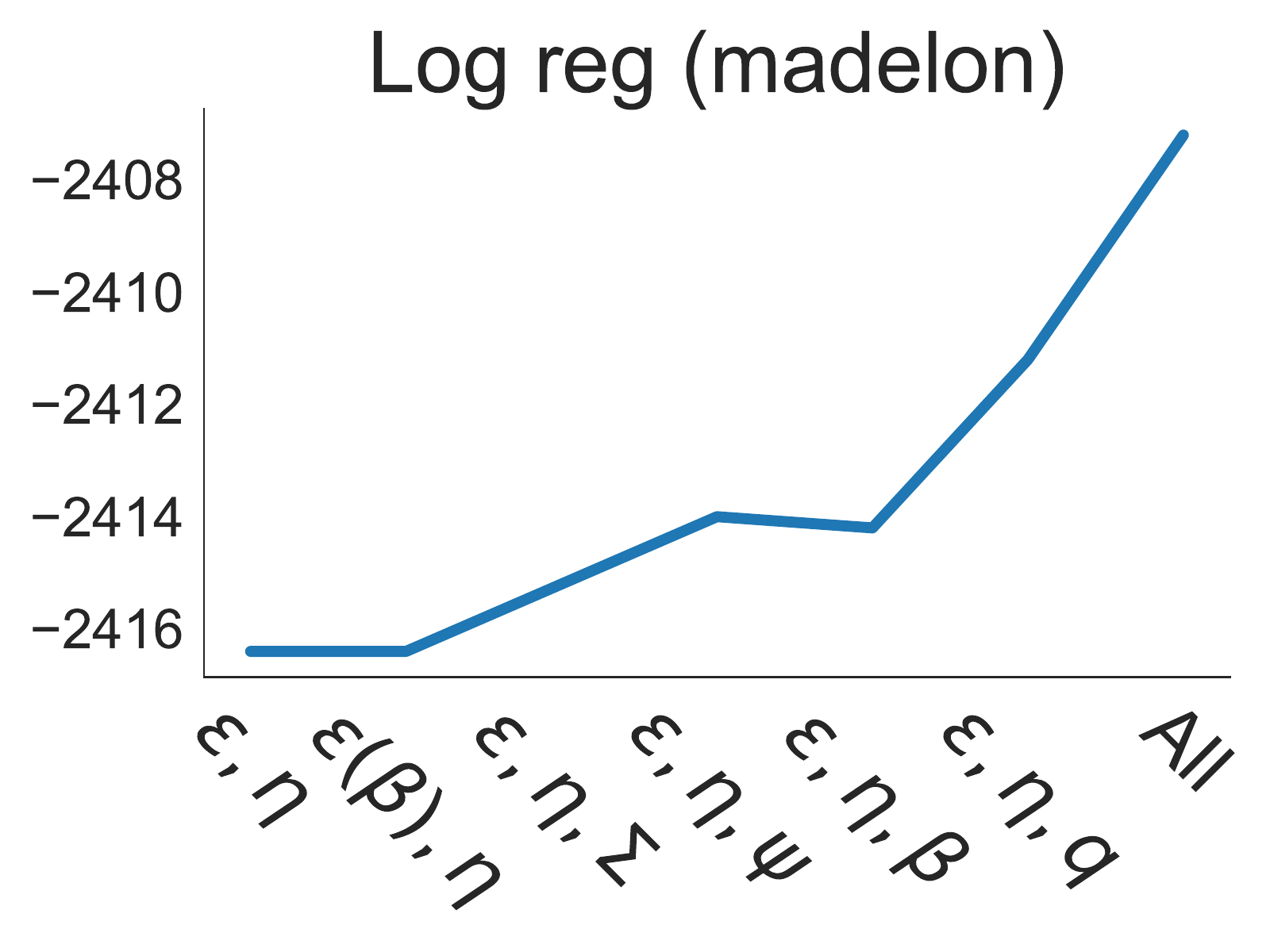}
  \caption{\textbf{Tuning all parameters leads to better results than tuning subsets of them. Largest gains are obtained by tuning bridging coefficients $\beta$ and initial distribution $q$.} ELBO achieved as a function of parameters tuned (x-axis), for $K = 64$. The subsets are ordered in terms of increasing performance (same ordering is used for all four models). Parameters are step-size $\epsilon$, damping coefficient $\eta$, moment covariance $\Sigma$, bridging densities parameters $\beta$ and $\psi$, initial distribution $q$.}
  \label{fig:tuning_more_1_fig}
\end{figure}

Finally, Appendix \ref{app:comparehmc} shows results comparing \UHA\ (tuning several parameters) against HMC, mean field VI and IW in terms of the approximation accuracy achieved on a logistic regression model with a fixed computational budget.

\subsection{VAE training}

Our method can be used to train latent variable models, such as Variational Auto-encoders (VAE) \cite{vaes_welling, rezende2014stochastic}. In this case the initial approximation $q(z|x)$ and the model $p(x, z)$ are parameterized by two neural networks (encoder and decoder), whose parameters are trained by maximizing the ELBO. \UHA\ can be used to train VAEs by augmenting these two distributions as described in Section \ref{sec:uhais}.

\textbf{Datasets.} We use three datasets: mnist \cite{lecun1998gradient} (numbers 1-9), emnist-letters \cite{cohen2017emnist} (letters A-Z), and kmnist \cite{clanuwat2018deep} (cursive Kuzushiji). All consist on greyscale images of $28\times 28$ pixels. In all cases we use stochastic binarization \cite{salakhutdinov2008quantitative} and a training set of $50000$ samples, a validation set of $10000$ samples, and a test set of $10000$ samples. All datasets are available in tensorflow-datasets \cite{TFDS}.

\textbf{Baselines.} We compare against Importance Weighted Auto-encoders \cite{IWVAE} and plain VAE training \cite{vaes_welling}.

% \textbf{Architecture details.} We set the encoder and decoder to be neural networks with one hidden layer of size $450$ with \textit{Relu} non-linearity, and the latent space dimensionality to $64$. We set $q(z\vert x)$ to a diagonal Gaussian whose mean and log-variance are given by the output of the encoder, $p(z)$ to a standard Normal, and $p(x\vert z)$ to a Bernoulli whose parameters are given by the output of the decoder.

\textbf{Architecture details.} We set $q(z|x)$ to a diagonal Gaussian, $p(z)$ to a standard Normal, and $p(x\vert z)$ to a Bernoulli. We consider two architectures for the encoder and decoder: (1) Feed forward networks with one hidden layer of size $450$ and Relu non-linearities, with a latent space dimensionality of $64$; (2) Architecture used by Burda et al. \cite{IWVAE}, feed forward networks with two hidden layers of size $200$ with tanh non-linearities, with a latent space dimensionality of $50$.

% \textbf{Architecture details.} The encoder and decoder are neural networks with one hidden layer of size $450$ with \textit{Relu} non-linearity, and the latent space dimensionality is $64$. We set $q(z\vert x)$ to a diagonal Gaussian parameterized the encoder, $p(z)$ to a standard Normal, and $p(x\vert z)$ to a Bernoulli parameterized by the decoder.

\textbf{Training details.} In all cases the encoder and decoder are initialized to parameters that maximize the ELBO. For IW we tune the encoder and decoder parameters (using the doubly-reparameterized estimator \cite{doublyrep}), and for UHA we tune the integration step-size $\epsilon$, damping coefficient $\eta$, bridging parameters $\beta$, momentum covariance $\Sigma$ (diagonal), and the decoder parameters. Following Caterini et al. \cite{caterini2018hamiltonian} we constrain $\epsilon \in (0, 0.05)$ to avoid unstable behavior of the leapfrog discretization. We use Adam with a step-size of $10^{-4}$ to train for $100$ epochs and use the validation set for early stopping. We repeated all simulations for three different random seeds. In all cases the standard deviation of the results was less than $0.1$ nats (not shown in tables).

All methods achieved better results using the architecture with one hidden layer. These results are shown in Tables \ref{table:resultsVAEELBO} and \ref{table:resultsVAENMLL}. The first one shows the ELBO on the test set achieved for different values of $K$, and the second one the log-likelihood on the test set estimated with AIS \cite{wu2016quantitative}. It can be observed that \UHA\ leads to higher ELBOs, higher log-likelihoods, and smaller variational gaps (difference between ELBO and log-likelihood) than IW for all datasets, with the difference between both methods' performance increasing for increasing $K$. Notably, for $K = 64$, the variational gap for \UHA\ becomes quite small, ranging from $0.8$ to $1.4$ nats depending on the dataset.

% Results for the architecture from Burda et al. \cite{IWVAE} (two hidden layers) are shown in Tables \ref{table:resultsVAEELBOA} and \ref{table:resultsVAENMLLA} (Appendix \ref{app:resultsIWAE}). Overall, results with this architecture are worse for both methods. For the emnist-letters and kmnist datasets both IW and \UHA\ achieve a lower log-likelihood and ELBO on the test set (with similar training ELBO, which may indicate overfitting). Despite this, we observe again that \UHA\ leads to higher ELBOs and smaller variational gaps for all values of $K$. For this architecture, for all datasets, the best test log-likelihood was achieved by \UHA\ with $K = 64$. However, for smaller $K$, IW sometimes yields models with higher log-likelihoods (despite lower ELBOs).

Results for the architecture from Burda et al. \cite{IWVAE} (two hidden layers) are shown in Tables \ref{table:resultsVAEELBOA} and \ref{table:resultsVAENMLLA} (Appendix \ref{app:resultsIWAE}). Again, we observe that \UHA\ consistently leads to higher ELBOs and the best test log-likelihood was consistently achieved by UHA with $K = 64$. However, for smaller $K$, IW sometimes had better log-likelihoods than UHA (despite worse ELBOs).

% ELBO AIS [7, 15, 31, 63] [-133.54779, -132.31235, -131.48083, -130.86096]
% ELBO IW [1, 8, 16, 32, 64] [-137.99034, -134.58394, -133.85295, -133.26155, -132.74828]

% ELBO AIS [7, 15, 31, 63] [-176.57231, -174.62634, -173.22827, -171.63672]
% ELBO IW [1, 8, 16, 32, 64] [-184.24092, -179.67766, -178.6613, -177.77759, -177.02754]

% ELBO AIS [7, 15, 31, 63] [-89.84434, -88.82424, -88.09347, -87.62214]
% ELBO IW [1, 8, 16, 32, 64] [-93.40421, -90.52832, -89.92981, -89.424805, -89.010414]

\renewcommand{\arraystretch}{1}
\begin{table}[ht]
  \caption{ELBO on the test set (higher is better). For $K = 1$ both methods reduce to plain VI.}
  \label{table:resultsVAEELBO}
  \centering
  \begin{tabular}{lllllll}
    \toprule
    & & $K = 1$ & $K = 8$ & $K = 16$ & $K = 32$ & $K = 64$ \\
    \midrule
    \multirow{2}{1cm}{mnist}  & UHA & $-93.4$ & $-89.8$ & $-88.8$ & $-88.1$ & $-87.6$ \\
                              & IW  & $-93.4$ & $-90.5$ & $-89.9$ & $-89.4$ & $-89.0$ \\
    \midrule
    \multirow{2}{1cm}{letters} & UHA & $-137.9$ & $-133.5$ & $-132.3$ & $-131.5$ & $-130.9$ \\
                               & IW  & $-137.9$ & $-134.6$ & $-133.9$ & $-133.2$ & $-132.7$ \\
    \midrule
    \multirow{2}{1cm}{kmnist}  & UHA & $-184.2$ & $-176.6$ & $-174.6$ & $-173.2$ & $-171.6$ \\
                               & IW  & $-184.2$ & $-179.7$ & $-178.7$ & $-177.8$ & $-177.0$ \\
    \bottomrule
  \end{tabular}
\end{table}

% LL AIS [7, 15, 31, 63] [-130.68892, -130.35402, -130.14122, -129.94914]
% LL IW [1, 8, 16, 32, 64] [-131.85773, -130.8955, -130.71538, -130.5681, -130.43288]

% LL AIS [7, 15, 31, 63] [-172.17339, -171.60742, -171.19075, -170.17998]
% LL IW [1, 8, 16, 32, 64] [-174.28265, -172.9789, -172.69804, -172.41594, -172.2162]

% LL AIS [7, 15, 31, 63] [-87.529884, -87.24827, -87.03542, -86.91007]
% LL IW [1, 8, 16, 32, 64] [-88.47396, -87.62502, -87.48456, -87.34197, -87.23049]

% \vspace{-0.5cm}
\begin{table}[ht]
  \caption{Log-likelihood on the test set (higher is better). This is estimated using AIS with under-damped HMC using $2000$ bridging densities, $1$ HMC iteration with $16$ leapfrog steps per bridging density, integration step-size $\epsilon = 0.06$, and damping coefficient $\eta = 0.8$.}
  \label{table:resultsVAENMLL}
  \centering
  \begin{tabular}{lllllll}
    \toprule
    & & $K = 1$ & $K = 8$ & $K = 16$ & $K = 32$ & $K = 64$ \\
    \midrule
    \multirow{2}{1cm}{mnist}  & UHA & $-88.5$ & $-87.5$ & $-87.2$ & $-87.0$ & $-86.9$ \\
                              & IW  & $-88.5$ & $-87.6$ & $-87.5$ & $-87.3$ & $-87.2$ \\
    \midrule
    \multirow{2}{1cm}{letters} & UHA & $-131.9$ & $-130.7$ & $-130.3$ & $-130.1$ & $-129.9$ \\
                               & IW  & $-131.9$ & $-130.9$ & $-130.7$ & $-130.6$ & $-130.4$ \\
    \midrule
    \multirow{2}{1cm}{kmnist}  & UHA & $-174.3$ & $-172.2$ & $-171.6$ & $-171.2$ & $-170.2$ \\
                               & IW  & $-174.3$ & $-173.0$ & $-172.6$ & $-172.4$ & $-172.2$ \\
    \bottomrule
  \end{tabular}
\end{table}

% \renewcommand{\arraystretch}{1.1}
% \begin{table}[ht]
%   \caption{ELBO achieved by tuning different subsets of parameters for $K = 64$. All subsets of parameters considered include the integration step-size $\epsilon$ and damping coefficient $\eta$. The subsets are ordered in terms of increasing performance (this ordering coincides for all four models).}
%   \label{table:tuningind}
%   \centering
%   \begin{tabular}{llllllll}
%     \toprule
%     \multirow{2}{1cm}{Model} & \multicolumn{7}{c}{Parameters tuned}\\
%     \cmidrule{2-8}
%     & \{$\epsilon, \eta$\} & $\{\epsilon(\beta), \eta\}$ & $\{\epsilon, \eta, \momcov\}$
%     & $\{\epsilon, \eta, \psi\}$ & $\{\epsilon, \eta, \beta\}$ & $\{\epsilon, \eta, q\}$ & All \\
%     \midrule
%     Brownian      & $-0.42$   & $-0.41$   & $-0.29$   & $-0.29$   & $-0.25$   & $-0.1$    & $0.23$ \\
%     Lorenz        & $-1160.9$ & $-1160.9$ & $-1160.8$ & $-1159.9$ & $-1159.5$ & $-1159.0$ & $-1158.9$ \\
%     LR (a1a)      & $-696.8$  & $-696.8$  & $-696.2$  & $-693.1$  & $-692.7$  & $-691.6$  & $-686.7$ \\
%     LR (madelon)  & $-2416.4$ & $-2416.4$ & $-2415.2$ & $-2414.0$ & $-2414.2$ & $-2411.2$ & $-2407.2$ \\
%     \bottomrule
%   \end{tabular}
% \end{table}

\section{Discussion}

Since \UHA\ yields a differentiable lower bound, one could tune other parameters not considered in this work. For instance, a different momentum distribution per bridging density could be used, that is, $\unnorm \pi_m(z, \rho) = \unnorm \pi_m(z) \MD_m(\rho)$. We believe additions such as this may yield further gains.
Also, our method can be used to get tight and differentiable upper bounds on $\log Z$ using the reversed AIS procedure described by Grosse et al. \cite{grosse2015sandwiching}.

Finally, removing accept-reject steps might sometimes lead to instabilities during optimization if the step-size $\epsilon$ becomes large. We observed this effect when training VAEs on some datasets for the larger values of $K$. We solved this by constraining the range of $\epsilon$ (previously done by Caterini et al. \cite{caterini2018hamiltonian}). While this simple solution works well, we believe that other approaches (e.g. regularization, automatic adaptation) could work even better. We leave the study of such alternatives for future work.

% In addition, our method could also be used to get tighter upper bounds on $\log Z$ using the reversed AIS procedure described by Grosse et al. \cite{grosse2015sandwiching}. Their method to get an upper bound is similar to regular AIS, but starts with a sample from the true target (instead of the initial approximation $q$). One could easily derive the "reversed" version of \UHA\ to produce powerful and differentiable upper bounds.

\begin{ack}
This material is based upon work supported in part by the National Science Foundation under Grant No. 1908577.
% Use unnumbered first level headings for the acknowledgments. All acknowledgments
% More information about this disclosure can be found at: \url{https://neurips.cc/Conferences/2021/PaperInformation/FundingDisclosure}.
% Do {\bf not} include this section in the anonymized submission, only in the final paper. You can use the \texttt{ack} environment provided in the style file to autmoatically hide this section in the anonymized submission.
\end{ack}

% \jd{In the references, make sure any arxiv papers have their "real" publication. I think this is true for [11] }

% \clearpage
\bibliography{control}
\bibliographystyle{plain}
\clearpage
\newpage

%!TEX root = ../main.tex

\appendix

\section{Generating the Hamiltonian AIS bound} \label{app:HAISboundalg}

\begin{algorithm}[H]
\caption{Generating the (non-differentiable) Hamiltonian AIS variational bound.}
\label{alg:sampleais}
\begin{algorithmic}
\State Sample $z_1 \sim q$ and $\rho_1 \sim \MD$.
\State Initialize estimator as $\mathcal{L} \leftarrow -\log q(z_1, \rho_1)$.
\For{$m = 1, 2, \cdots , M-1$}
	\State Run corrected $T_m$ (Alg.~\ref{alg:correctedtm}) on input $(z_m, \rho_m)$, storing the output $(z_{m+1}, \rho_{m+1})$.
	\State Update estimator as $\mathcal{L} \leftarrow \mathcal{L} + \log \left( \unnorm \pi_m(z_m, \rho_m) / \unnorm \pi_m(z_{m+1}, \rho_{m+1}) \right)$.
\EndFor
\State Update estimator as $\mathcal{L} \leftarrow \mathcal{L} + \log \unnorm p(z_M, \rho_M)$.
\State \Return $\mathcal{L}$
\end{algorithmic}
\end{algorithm}

\section{More results tuning more subsets of parameters for \UHA} \label{app:tunemoreuha}

We tested \UHA\ tuning different subsets of $\{\epsilon, \eta, \momcov, \beta, q(z), \epsilon(\beta), \psi(\beta)\}$. Fig.~\ref{fig:tuning_more} shows the results. The first row shows the results obtained by tuning the pair $(\epsilon, \eta)$ and each other parameter individually for different values of $K$, and the second row shows the results obtained by tuning increasingly more parameters. It can be observed that tuning $\beta$ and $q(z)$ lead to the largest gains in performance.

\begin{figure}[ht]
  \centering
  \includegraphics[scale=0.4, trim = {0 1.5cm 0 0}, clip]{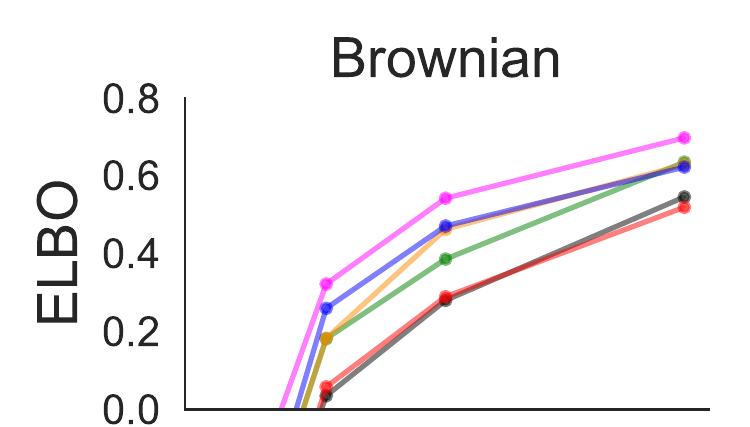}\hfill
  \includegraphics[scale=0.4, trim = {1.1cm 1.5cm 0 0}, clip]{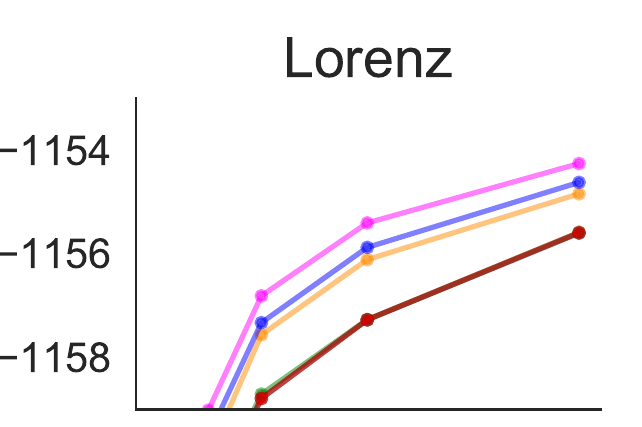}\hfill
  \includegraphics[scale=0.4, trim = {1.3cm 1.5cm 0 0}, clip]{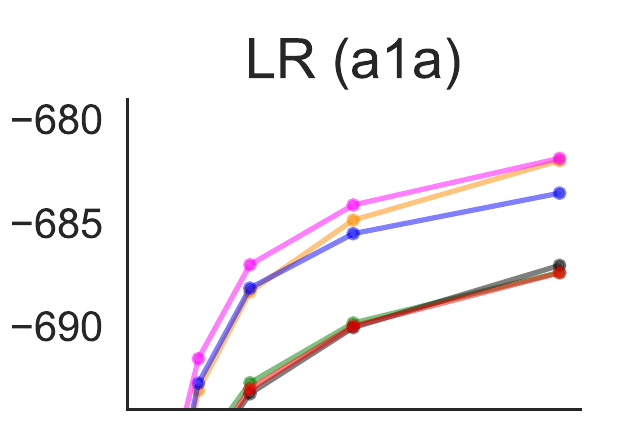}\hfill
  \includegraphics[scale=0.4, trim = {1.3cm 1.5cm 0 0}, clip]{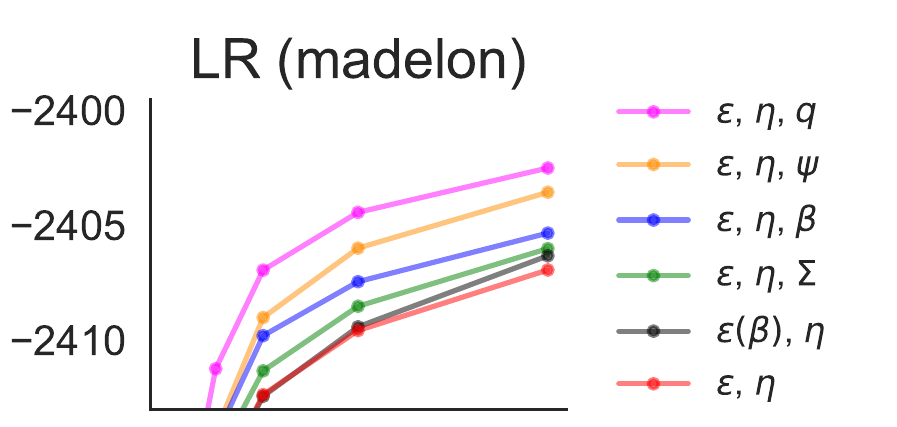}
  \hspace{0.8cm}

  \vspace{0.2cm}

  \includegraphics[scale=0.4, trim = {0 0 0 0.86cm}, clip]{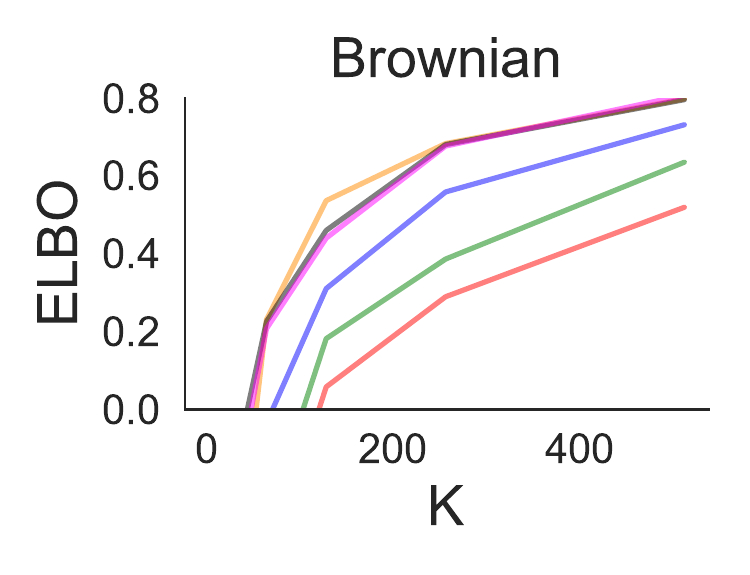}\hfill
  \includegraphics[scale=0.4, trim = {1.1cm 0 0 0.87cm}, clip]{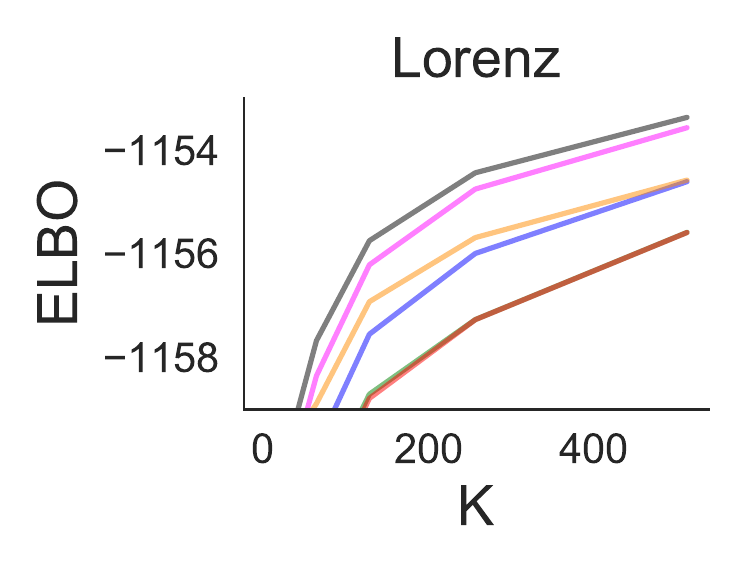}\hfill
  \includegraphics[scale=0.4, trim = {1.3cm 0 0 0.9cm}, clip]{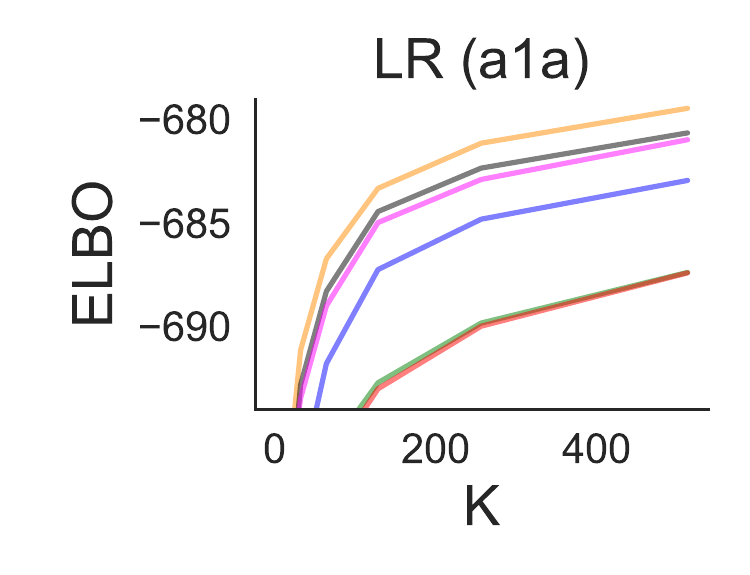}\hfill
  \includegraphics[scale=0.4, trim = {1.3cm 0 0 0.9cm}, clip]{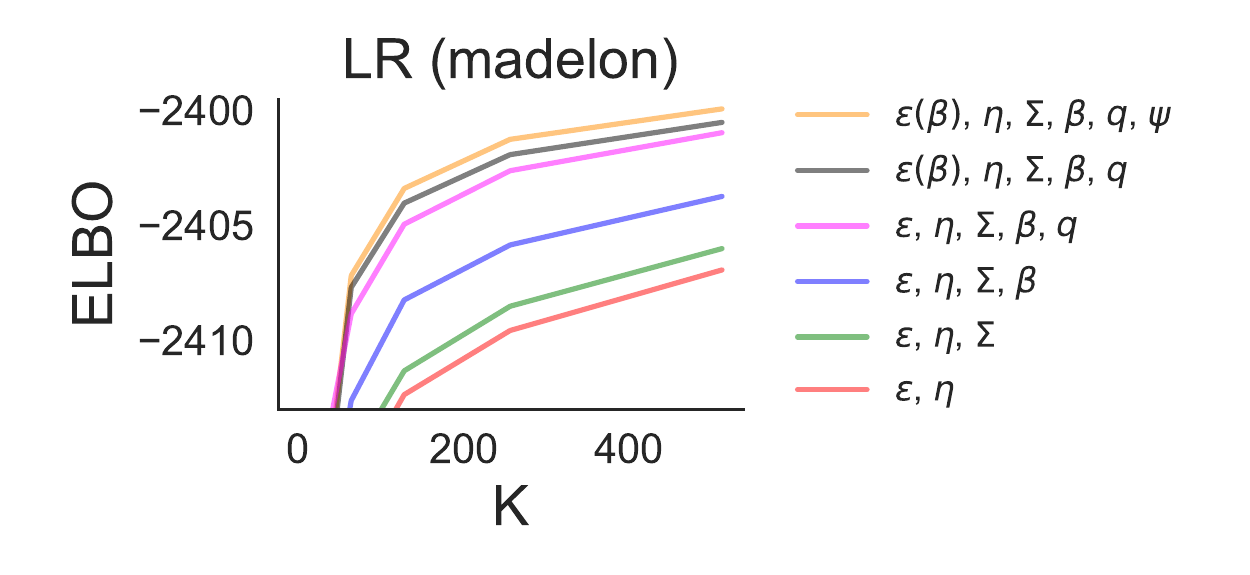}
  \caption{\textbf{Tuning more parameters leads to significantly better results.} Legends indicate what parameters are being trained. Parameters are step-size $\epsilon$, damping coefficient $\eta$, moment covariance $\Sigma$, bridging densities parameters $\beta$ and $\psi$, initial distribution $q$. $\epsilon(\beta)$ indicates we are learning the step-size as an affine function of $\beta$.}
  \label{fig:tuning_more}
\end{figure}

\section{Results using architecture from Burda et al. \cite{IWVAE}} \label{app:resultsIWAE}

In this section we show the results achieved for VAE training using the architecture from Burda et al. \cite{IWVAE} (with $1$ stochastic layer). In this case the encoder and decoder consist on feed forward neural networks with two hidden layers of size $200$ with \textit{Tanh} non-linearity, and latent space dimensionality of $50$. All training details are the same, but with the constraint $\epsilon \in (0, 0.04)$. Tables \ref{table:resultsVAEELBOA} and \ref{table:resultsVAENMLLA} show the results.

% ELBO AIS [7, 15, 31, 63] [-134.31573, -133.27618, -132.56648, -131.16951]
% ELBO IW [1, 8, 16, 32, 64] [-139.03172, -135.54303, -134.72406, -134.0166, -133.42055]

% ELBO AIS [7, 15, 31, 63] [-189.48364, -188.0871, -187.088, -180.28703]
% ELBO IW [1, 8, 16, 32, 64] [-197.45915, -191.81526, -190.19653, -188.8296, -187.60118]

% ELBO AIS [7, 15, 31, 63] [-89.17246, -88.53774, -88.07794, -87.11738]
% ELBO IW [1, 8, 16, 32, 64] [-92.36892, -89.88202, -89.31567, -88.84015, -88.459656]

\renewcommand{\arraystretch}{1}
\begin{table}[ht]
  \caption{ELBO on the test set (higher is better). For $K = 1$ both methods reduce to plain VI.}
  \label{table:resultsVAEELBOA}
  \centering
  \begin{tabular}{lllllll}
    \toprule
    & & $K = 1$ & $K = 8$ & $K = 16$ & $K = 32$ & $K = 64$ \\
    \midrule
    \multirow{2}{1cm}{mnist}  & UHA & $-92.4$ & $-89.2$ & $-88.5$ & $-88.1$ & $-87.1$ \\
                              & IW  & $-92.4$ & $-89.9$ & $-89.3$ & $-88.8$ & $-88.5$ \\
    \midrule
    \multirow{2}{1cm}{letters} & UHA & $-139.0$ & $-134.3$ & $-133.3$ & $-132.6$ & $-131.2$ \\
                               & IW  & $-139.0$ & $-135.5$ & $-134.7$ & $-134.0$ & $-133.4$ \\
    \midrule
    \multirow{2}{1cm}{kmnist}  & UHA & $-197.5$ & $-189.5$ & $-188.1$ & $-187.1$ & $-180.3$ \\
                               & IW  & $-197.5$ & $-191.8$ & $-190.2$ & $-188.8$ & $-187.6$ \\
    \bottomrule
  \end{tabular}
\end{table}

% LL AIS [7, 15, 31, 63] [-131.78796, -131.44426, -131.23482, -129.83711]
% LL IW [1, 8, 16, 32, 64] [-133.03381, -131.6323, -131.23698, -130.88538, -130.6211]

% LL AIS [7, 15, 31, 63] [-186.28233, -185.76295, -185.3224, -177.36072]
% LL IW [1, 8, 16, 32, 64] [-188.27483, -184.3666, -183.16008, -182.13933, -181.21844]

% LL AIS [7, 15, 31, 63] [-87.58126, -87.39932, -87.26233, -86.2587]
% LL IW [1, 8, 16, 32, 64] [-88.2803, -87.30457, -87.007744, -86.77047, -86.59281]

\begin{table}[ht]
  \caption{Log-likelihood on the test set (higher is better). This is estimated using AIS with under-damped HMC using $2000$ bridging densities, $1$ HMC iteration with $16$ leapfrog steps per bridging density, integration step-size $\epsilon = 0.05$, and damping coefficient $\eta = 0.8$.}
  \label{table:resultsVAENMLLA}
  \centering
  \begin{tabular}{lllllll}
    \toprule
    & & $K = 1$ & $K = 8$ & $K = 16$ & $K = 32$ & $K = 64$ \\
    \midrule
    \multirow{2}{1cm}{mnist}  & UHA & $-88.3$ & $-87.6$ & $-87.4$ & $-87.3$ & $-86.3$ \\
                              & IW  & $-88.3$ & $-87.3$ & $-87.0$ & $-86.8$ & $-86.6$ \\
    \midrule
    \multirow{2}{1cm}{letters} & UHA & $-133.0$ & $-131.8$ & $-131.4$ & $-131.2$ & $-129.9$ \\
                               & IW  & $-133.0$ & $-131.6$ & $-131.2$ & $-130.9$ & $-130.6$ \\
    \midrule
    \multirow{2}{1cm}{kmnist}  & UHA & $-188.3$ & $-186.3$ & $-185.8$ & $-185.3$ & $-177.4$ \\
                               & IW  & $-188.3$ & $-184.4$ & $-183.2$ & $-182.1$ & $-181.2$ \\
    \bottomrule
  \end{tabular}
\end{table}

\section{Extrapolating optimal parameters for \UHA} \label{sec:extrapol}

Some results in Section \ref{sec:tuningmore} (and Appendix \ref{app:tunemoreuha}) use a number of bridging densities $K$ up to 512. As mentioned previously, for those simulations, if $K_1 \geq 64$ bridging densities were used, we optimized the parameters for $K_2 = 64$ and extrapolate the parameters to work with $K_1$. We now explain this procedure.

From the parameters considered, $\{\epsilon, \eta, \momcov, \beta, q(z), \epsilon(\beta), \psi(\beta)\}$, the only ones that need to be "extrapolated" are the step-size $\epsilon$ and the bridging parameters $\beta$. All other parameters are tuned for $K_2 = 64$ bridging densities and the values obtained are directly used with $K_1$ bridging densities. 

For $\beta$ we use a simple interpolation. Define $f(x)$ to be the piecewise linear function (with $K_2 = 64$ "pieces") that satisfies $f(x_k) = \beta_k$, for $x_k = k / K_2$ and $k = 0, \cdots, K_2$ (this is a bijection from $[0, 1]$ to $[0, 1]$). Then, when using $K_1$, we simply define $\beta_k = f(x_k)$, where $x_k = k / K_1$ and $k = 0, \cdots, K_1$.

For $\epsilon$, we use the transformation $\epsilon_{K_1} = \epsilon_{K_2} \frac{\log K_2}{\log K_1}$. While other transformations could be used (e.g. without the $\log$), we observed this to work best in practice. (In fact, we obtained this rule by analyzing the dependence of the optimal $\epsilon$ on $K$ for several tasks and values of $K$.)

\section{Approximation accuracy} \label{app:comparehmc}

We study the accuracy of the approximation provided by UHA by analyzing the posterior moment errors: We estimate the mean and covariance of the target distribution using UHA and compute the mean absolute error of these estimates. (We get the ground truth values using approximate samples obtained running NUTS \cite{hoffman2014no} for $500000$ steps.) We consider a logistic regression model with the \textit{sonar} dataset ($d = 61$), and compare against mean field VI, IW, and HMC. We give each method the same computational budget $B$, measured as the total number of model evaluations (or gradient), and perform simulations for $B\in \{10^5, 5\times 10^5, 10^6\}$.

For HMC, we use half of the budget for the warm-up phase and half to draw samples. For mean field VI we use the whole budget for optimization, and use the final mean and variance parameters for the approximation. For \UHA\ and IW we train using $K=32$ for $3000$ steps, and use the remaining budget of model evaluations to draw samples (used to estimate posterior moments) using $K=256$.\footnote{For \UHA\ we use the extrapolation explained in Appendix \ref{sec:extrapol}} For \UHA\ we tune the step-size $\epsilon$, the damping coefficient $\eta$, the momentum distribution covariance (diagonal), the bridging densities coefficients $\beta$, and the parameters of the initial distribution $q(z)$.

Fig. \ref{fig:posterior_approx} shows the results for the posterior covariance. We do not include the results for the posterior mean because all methods perform similarly. It can be observed that HMC achieves the lowest error, followed by \UHA. Both mean field VI and IW yield significantly worse results.

\begin{figure}[ht]
  \centering
  \includegraphics[scale=0.4, trim = {0 0 0 0}, clip]{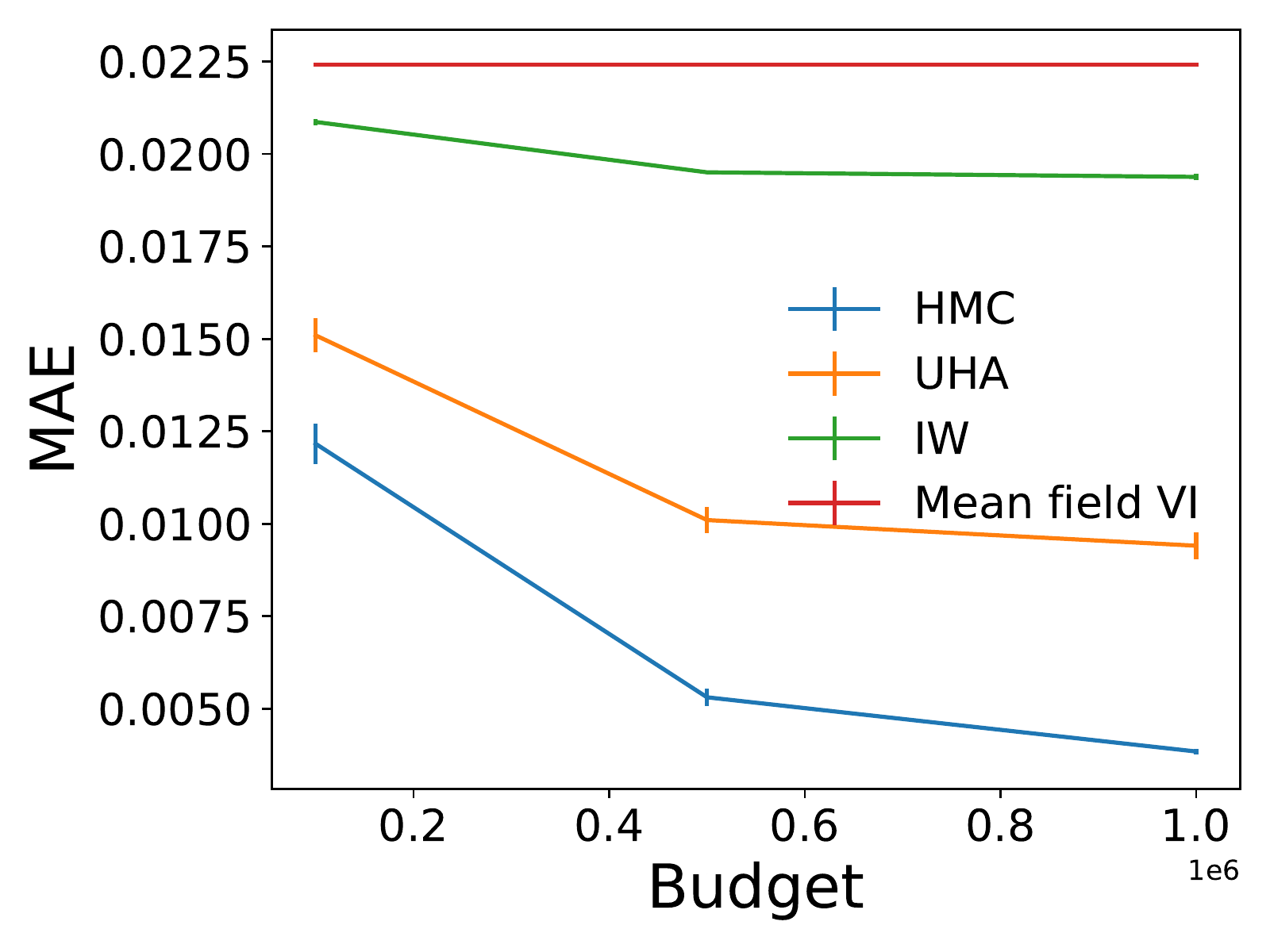}
  \caption{Mean absolute error for posterior covariance approximation. Standard errors computed by repeating the simulations using five different random seeds.}
  \label{fig:posterior_approx}
\end{figure}

\section{Proof of Lemma \ref{lemma:reversal}} \label{app:proofrev}

We begin with the following result.

\begin{lemma}
Let $T_1$, $T_2$ and $T_3$ be three transitions that leave some distribution $\pi$ invariant and satisfy $T_i(z'\vert z) \pi(z) = U_i(z\vert z') \pi(z')$ (i.e. $U_i$ is the reversal of $T_i$ with respect to $\pi$). Then the reversal of $T$ with respect to $\pi$ is given by $U = U_3 \circ U_2 \circ U_1$.
\end{lemma}

\begin{proof}
\begin{align}
T(z'\vert z) \pi(z) & = \int T_3(z' \vert z_2) T_2(z_2 \vert z_1) T_1(z_1 \vert z) \pi(z) \,dz_1\,dz_2\\
& = \int T_3(z' \vert z_2) T_2(z_2 \vert z_1) \pi(z_1 ) U_1(z \vert z_1) \,dz_1\,dz_2\\
& = \int T_3(z' \vert z_2) \pi(z_2) U_2(z_1 \vert z_2) U_1(z \vert z_1) \,dz_1\,dz_2\\
& = \pi(z') \int U_3(z_2 \vert z') U_2(z_1 \vert z_2) U_1(z \vert z_1) \,dz_1\,dz_2\\
& = \pi(z') \int U_1(z \vert z_1) U_2(z_1 \vert z_2) U_3(z_2 \vert z') \,dz_1\,dz_2\\
& = \pi(z') U(z \vert z').
\end{align}
\end{proof}

The rest of the proof is straightforward. Let the three steps from the corrected version of $T_m$ (Alg.~\ref{alg:correctedtm}) be denoted $T_m^1$, $T_m^2$ and $T_m^3$. The latter two (Hamiltonian simulation with accept-reject step and momentum negation) satisfy detailed balance with respect to $\pi_m(z, \rho)$ \cite[\S3.2]{neal2011mcmc}. Thus, for these two, $U_m^i$ is defined by the same process as $T_m^i$. For $T_m^1$ (momentum resampling), its reversal is given by the reversal of $s(\rho'\vert \rho)$ with respect to $\MD(\rho)$. We call this $s_{\mathrm{rev}}(\rho\vert \rho')$, and it satisfies
\begin{equation}
s_{\mathrm{rev}}(\rho\vert \rho') = s(\rho'\vert \rho) \frac{\MD(\rho)}{\MD(\rho')}.
\end{equation}

\section{Proof of Theorem \ref{thm:ratioUT}} \label{app:proofthm}

To deal with delta functions, whenever the transition states [Set $z' \leftarrow z$], we use $z' \sim \mathcal{N}(z, a)$, and take the limit $a \to 0$. We use $g_a(z)$ to denote the density of a Gaussian with mean zero and variance $a$ evaluated at $z$, and $\gamma(z, \rho) = (z, -\rho)$ (operator that negates momentum).

We first compute $T_m(z_{m+1}, \rho_{m+1} \vert z_m, \rho_m)$. We have that $\rho'_m \sim s(\cdot\vert \rho_m)$ and $z'_m \sim \mathcal{N}(z_m, a)$. Thus, \begin{equation}
T_m'(z_{m}', \rho_{m}' \vert z_m, \rho_m) = s(\rho_m'\vert \rho_m) g_a(z'_m - z_m).
\end{equation}

Also, we have $(z_{m+1}, \rho_{m+1}) = (\gamma \circ \mathcal{T}_m)(z'_m, \rho'_m)$. Since $\gamma \circ \mathcal{T}_m$ is an invertible transformation with unit Jacobian and inverse $(\gamma \circ \mathcal{T}_m)^{-1} = \mathcal{T}_m \circ \gamma$, we get that
\begin{align}
T_m(z_{m+1}, \rho_{m+1} \vert z_m, \rho_m) & = T_m'((\mathcal{T}_m \circ \gamma) \, (z_{m+1}, \rho_{m+1}) \vert z_m, \rho_m)\\
& = s(\mathcal{T}_m^\rho(z_{m+1}, -\rho_{m+1})\vert \rho_m) \, g_a(\mathcal{T}_m^z(z_{m+1}, -\rho_{m+1}) - z_m),
\end{align}
where $\mathcal{T}_m^\rho$ is the operator that applies $\mathcal{T}_m$ and returns the second component of the result (and similarly for $\mathcal{T}_m^z$).

Now, we compute $U_m(z_m, \rho_m \vert z_{m+1}, \rho_{m+1})$. We have that $(z'_m, \rho'_m) = (\mathcal{T}_m \circ \gamma) \, (z_{m+1}, \rho_{m+1})$. Thus,
\begin{align}
U_m(z_m, \rho_m \vert z_{m+1}, \rho_{m+1}) & = U_m(z_m, \rho_m \vert z_{m}', \rho_{m}')\\
& = s_\mathrm{rev}(\rho_m \vert \rho_m') \, g_a(z_m - z_m')\\
& = s_\mathrm{rev}(\rho_m \vert \mathcal{T}_m^\rho(z_{m+1}, -\rho_{m+1}) \, g_a(z_m - \mathcal{T}_m^z(z_{m+1}, -\rho_{m+1})).
\end{align}

Taking the ratio $U_m(z_m, \rho_m \vert z_{m+1}, \rho_{m+1}) / T_m(z_{m+1}, \rho_{m+1} \vert z_m, \rho_m)$ the factors involving the Gaussian pdf cancel (the density of a mean zero Gaussian is symmetric) and using that
\begin{equation}
s_\mathrm{rev}(\rho_m\vert \rho'_m) \MD(\rho'_m) = s(\rho'_m\vert \rho_m) \MD(\rho_m) \longrightarrow
\frac{s_\mathrm{rev}(\rho_m\vert \rho'_m)}{s(\rho'_m\vert \rho_m)} = \frac{\MD(\rho_m)}{\MD(\rho'_m)}
\end{equation}
yields get the desired result.

\end{document}